\documentclass[runningheads]{llncs}
\usepackage{graphicx}
\usepackage{amsmath,amssymb} 
\usepackage{color}
\usepackage{csquotes}

\usepackage{mathtools} 

\usepackage[width=122mm,left=12mm,paperwidth=146mm,height=193mm,top=12mm,paperheight=217mm]{geometry}
\newcommand{\etal}{et al.}
\newcommand{\eg}{e.g.}
\newcommand{\ie}{i.e.}
\usepackage[font=small,skip=0pt]{caption}
\usepackage{animate}

\usepackage{xr}
\makeatletter
\newcommand*{\addFileDependency}[1]{
  \typeout{(#1)}
  \@addtofilelist{#1}
  \IfFileExists{#1}{}{\typeout{No file #1.}}
}
\makeatother

\begin{document}

\pagestyle{headings}
\mainmatter

\title{Switchable Temporal Propagation Network} 

\titlerunning{Switchable Temporal Propagation Network}

\authorrunning{Sifei Liu \etal}

\author{Sifei Liu$^{1}$, Guangyu Zhong$^{1,3}$, Shalini De Melloç, Jinwei Gu$^{1}$ \\
Varun Jampani$^{1}$, Ming-Hsuan Yang$^{2}$, Jan Kautz$^{1}$}
\institute{$^{1}$NVIDIA, $^{2}$UC Merced, $^{3}$Dalian University of Technology}

\maketitle


\vspace{-1em}
\begin{abstract}
	%
	%
	%
	%
	%
Videos contain highly redundant information between frames.
Such redundancy has been extensively studied in video compression and encoding, but is less explored for more advanced video processing.
In this paper, we propose a learnable unified framework for propagating a variety of visual properties of video images, including but not limited to color, high dynamic range (HDR), and segmentation information, where the properties are available for only a few key-frames.
Our approach is based on a temporal propagation network (TPN), which models the transition-related affinity between a pair of frames in a purely data-driven manner.
We theoretically prove two essential factors for TPN: 
(a) by regularizing the global transformation matrix as orthogonal, the ``style energy'' of the property can be well preserved during propagation;
(b) such regularization can be achieved by the proposed switchable TPN with bi-directional training on pairs of frames.
We apply the switchable TPN to three tasks: colorizing a gray-scale video based on a few color key-frames, generating an HDR video from a low dynamic range (LDR) video and a few HDR frames, and propagating a segmentation mask from the first frame in videos. 
Experimental results show that our approach is significantly more accurate and efficient than the state-of-the-art methods.
All the codes and models will be released to the public.

\vspace{-1em}
\end{abstract}


%
\vspace{-2em}
\section{Introduction}
\label{sec:intro}
\vspace{-0.5em}

Videos contain highly redundant information between frames.
Consider a pair of consecutive frames randomly sampled from a video, it is likely that they are similar in terms of appearance, structure and content in most regions.
Such redundancy has been extensively studied in video compression to reduce the storage and speedup the transmission of videos, but is less explored for more advanced video processing.
%
%
%
A number of recent algorithms,  such as 
optical-flow based warping~\cite{gadde2017semantic}, similarity-guided filtering~\cite{he2013guided,levin2004colorization}
and the bilateral CNN model~\cite{Jampani17VPN}, explore   
the local relationships between frames to propagate information.
These methods model the similarity among pixels, regions or frames from either hand-crafted pixel-level features (e.g., pixel intensities and locations) or apparent motion (e.g., optical flow). 
They have several potential issues: 
(a) the designed similarity may not faithfully reflect the image structure, and
(b) such similarity may not express the high-level pairwise relationships between frames, e.g., for propagating a segmentation mask in the semantic domain.

\begin{figure*}[h]
	\centering
	\includegraphics[width=0.99\linewidth]{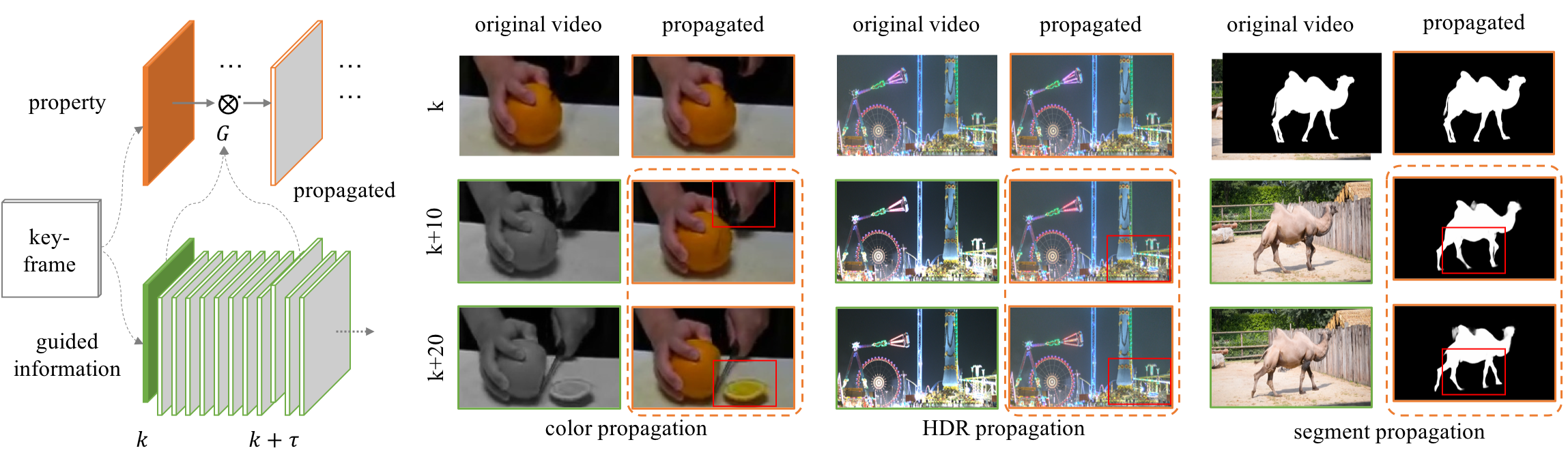}
	\caption{\footnotesize 
	We propose the TPN model that takes a known property map (e.g., color, HDR, segmentation mask) from a key-frame ($k$), and transform it to a nearby frame ($k+\tau$), denoted by ``propagated''.
	The propagation is guided by some known information (e.g., lightness, LDR, RGB image), via a learnable transformation matrix $G$.
	We show three tasks on the right size, where $k$ denotes the key-frame for which the ground-truth property is provided. The orange bboxes shows the propagated results by our algorithm, which are guided by the information from the left columns. 
	We highlight (red bboxes) the regions where the proposed method successfully deals with large transitions, or preserves fine details for the predicted property.  
	Zoom-in to see details.
		%
		%
		%
	}
	\label{figure:intro}\vspace{-9mm}
\end{figure*}

In this paper, we develop a temporal propagation network (TPN)  to explicitly learn pixel-level similarity between a pair of frames (see Fig.~\ref{figure:intro}).
It contains a propagation module that transfers a property (e.g., color) of one frame to a nearby frame through a global, linear transformation matrix which is learned with a CNN from any available guidance information (e.g., lightness).

We enforce two principles when learning propagation in the temporal domain: (a) \textbf{bi-directionality}, i.e., the propagation between a pair of frames should be invertible, and (b) \textbf{consistency}, i.e., the ``style energy'' (e.g., the global saturation of color) of the target property should be preserved during propagation.
We theoretically prove that:
enforcing both principles in the TPN is \textit{equivalent}
to ensuring that the transformation matrix is orthogonal with respect to each propagation direction.
This theoretical result allows us to implement TPN as a novel, special network architecture --- the switchable TPN (see Fig.~\ref{figure:stpn}) --- without explicitly solving the transformation matrix. 
It uses bi-directional training for a pair of frames in the propagation module, 
which is guided by switched output maps from the guidance CNN network.
%
%
Experiments demonstrate that the proposed architecture is effective in preserving the style energy even between two widely separated frames.

We validate the proposed model for three propagation tasks:
(a) video colorization from a few color key-frames and a grayscale video (Section~\ref{sec:color}).
With such temporal propagation, the workload of black-and-white video colorization can be largely reduced by only annotating a small number of key-frames.
(b) HDR video reconstruction from an LDR video with a few HDR key-frames (Section~\ref{sec:hdr}).
This is a new way for HDR video capture, where the whole video can be reconstructed with a few HDR frames provided.
(c) video segmentation when only the segmentation mask of the target in the first frame is provided.
We show that even without any image-based segmentation model, the proposed method can achieve comparable performance to the state-of-the-art algorithm.
All of these tasks reveal that video properties between temporally close frames are highly redundant, and that the relationships between them can be learned from corresponding guidance information.
Compared to the existing methods, and aside from the novel architecture, 
our proposed method also has the following advantages:
\vspace{-3mm}
\begin{itemize}
    \item \textbf{High accuracy.} 
    Compared to prior work~\cite{Jampani17VPN,Eilertsen17HDR},
    our TPN approach achieves significant improvement in video quality. More importantly, the switchable TPN preserves the style energy in the generated sequences significantly better than the network without the switchable structure. 
    \item \textbf{High efficiency.} Our method runs in real-time on a single Titan XP GPU for all the three tasks, which is about 30x to 50x faster over prior work~\cite{Jampani17VPN,Eilertsen17HDR} (see Table~\ref{tab:runtime}). Moreover, our model does not require sequential processing of video frames, \ie, all video frames  can be processed in parallel, which can further improve its efficiency.
\end{itemize}

\vspace{-5mm}
\section{Related Work and Problem Context}
\label{sec:related}

\vspace{-2mm}
\paragraph{\bf Modeling affinity for pixel propagation.}
Affinity is a generic measure of closeness between two pixels/entities 
and is widely used in vision tasks at all levels. 
Well-modeled affinity reveals how to propagate information from the known pixels to the unknown ones.
Most prior methods design affinity measures based on simple, intuitive functions~\cite{he2013guided,levin2008closed,levin2004colorization}. Recently, a deep CNN model is proposed to learn task-dependent affinity metric \cite{Liu17SPN} by modeling the propagation of pixels as an image diffusion process. 
%
%
%
While \cite{Liu17SPN} is limited to the spatial propagation of pixels for image segmentation, its high-level idea inspires us to learn pixel affinity in other domains via CNNs, e.g., in video sequences as proposed in this work.
%

Considerably less attention has been paid to developing methods for propagating temporal information across video frames.
Jampani~\etal~\cite{Jampani17VPN} propose to propagate video segmentation and color information by embedding pixels into a bilateral space~\cite{jampanicvpr2016} defined based on spatial, temporal and color information.
While pixels of the same region from different frames can be closer in such a space, it requires several previous frames stacked together to generate a new frame, which results in a high computational cost.
Our proposed algorithm is different in that it explicitly \emph{learns} the pixel affinity that describes the task-specific temporal frame transitions, instead of manually defining a similarity measure.

\vspace{-3mm}
\paragraph{\bf Colorizing grayscale videos}
Colorization in images and videos is achieved via an interactive procedure in~\cite{levin2004colorization}, which propagates manually annotated strokes spatially within or across frames, based on a matting Laplacian matrix and with manually defined similarities.
Recently, several methods based on CNNs have been developed to 
colorize pixels in images with fully-automatic or sparsely annotated colors~\cite{zhang2016colorful,zhang2017real}.
Due to the multinomial natural of color pixels~\cite{zhang2017real}, the interactive procedure usually gives better results. 
%
While interactive methods can be employed for single images, it is not practical to apply them for all frames of a monochrome video. 
In this work, we suggest a more plausible approach by using a few color key-frames to propagate visual information to all frames in between.
To this end, colorizing a full video can be easily achieved by annotating at sparse locations in only a few key-frames, as described in Section~\ref{sec:color}.

\vspace{-3mm}
\paragraph{\bf Video propagation for HDR imaging}
Most consumer-grade digital cameras have limited dynamic range and often capture
images with under/over-exposed regions, which not only 
degrades the quality of the captured photographs and videos, but also  
impairs the performance of computer vision tasks such as tracking and 
recognition in numerous applications.  
A common way to achieve HDR 
imaging is to capture a stack of LDR images with different exposures
and fuse them together~\cite{Debevec97HDR,Reinhard10HDR}. 
Such an approach assumes static scenes and thus requires deghosting techniques~\cite{Hu13HDR,Oh15HDR,Gallo15HDR} 
to remove artifacts. 
Capturing HDR videos for dynamic scenes poses a more challenging problem.
Prior methods to create HDR videos are mainly based on hardwares that either alternate exposures between
frames~\cite{Kang03HDR,Kalantari13HDR}, or use multiple cameras~\cite{Tocci11HDR} or specialized
image sensors with pixel-wise exposure controls~\cite{Nayar00HDR,Gu10CRSP}.
A few recent methods based on deep models have been developed for HDR imaging.
Kalantari \etal~\cite{Kalantari17HDR} use a deep neural network to align multiple LDR images into a single HDR image for dynamic scenes.  
Zhang \etal~\cite{Zhang17HDR} develop an auto-encoder network to predict a single HDR
panorama from a single exposed LDR image for image-based rendering.
In addition, Eilertsen~\etal~\cite{Eilertsen17HDR} propose a similar network for HDR reconstruction from a single LDR input image, which primarily focuses on recovering details in the high intensity saturated regions. 

In this paper, we apply the TPN for HDR video reconstruction from a LDR video.
Given a few HDR frames (or photos) and an LDR video, the TPN propagates the scene radiance information from the HDR frames (or photos) to the remaining frames in the LDR video. 
Note that unlike all the existing single LDR-based methods~\cite{Zhang17HDR,Eilertsen17HDR}, which hallucinate
the missing HDR details in images, we focus on propagating the HDR information from the input few HDR images to neighboring LDR frames.
It provides an alternative solution for efficient, low cost HDR video reconstruction.

\vspace{-3mm}
\section{Proposed Algorithm}
\label{sec:method}\vspace{-2mm}
We exploit the redundancy in videos, and propose the TPN for learning affinity and propagating target properties between frames. Take the video colorization as an example.
Given an old black-and-white movie with a few  key-frames colorized by artists, can we automatically colorize the entire movie?  
This problem can be equivalently reformulated as propagating a target property (\ie, color) based on the affinity of some features (\eg, lightness) between two frames. 
Intuitively, this is feasible because (1) videos have redundancy over time --- nearby frames tend to have similar appearance, and (2) the pixel correlation between two frames in the lightness domain is often consistent with that in the color domain. 
%
%
%
%

In this work, we model the propagation of a target property (\eg, color) between two frames as a \textit{linear transformation},
\begin{equation}
\label{eq:linear}
    U_{t} =  GU_k,
\end{equation}
where $U_k\in \mathcal{R}^{n^2\times 1}$ and $U_t\in \mathcal{R}^{n^2\times 1}$ are the vectorized version of the $n\times n$ property maps of a key-frame and a nearby frame, and $G\in \mathcal{R}^{n^2\times n^2}$ is the transformation matrix to be estimated\footnote{For property maps with multiple channels $n\times n\times c$, we treat each channel separately.}. 
Suppose we observe some features of the two frames (\eg, lightness) $V_k$ and $V_t$, the transformation matrix $G$ is thus a function of $V_k$ and $V_t$,
\begin{equation}
\label{eq:guidance}
    G = g(\theta, V_k, V_t).
\end{equation}

%
%
The matrix $G$ should be dense in order to model any type of pixel transition in a global scope, but $G$ should also be concise for efficient estimation and propagation. 
%
%
%
In Section~\ref{sec:spn}, we propose a solution, called the basic TPN,  by formulating the linear transformation $G$ as an image diffusion process similar to~\cite{Liu17SPN}.
%
%
Following that, in Section~\ref{sec:cycle}, we introduce the key part of our work, the switchable TPN, which enforces the bi-directionality and the style consistency for temporal propagation. 
%
We prove that enforcing these two principles is equivalent to ensure the transformation matrix $G$ to be orthogonal, which in turn can be easily implemented by equipping an ordinary temporal propagation network with a switchable structure.

\begin{figure*}[t]
	\centering
	\includegraphics[width=0.90\linewidth]{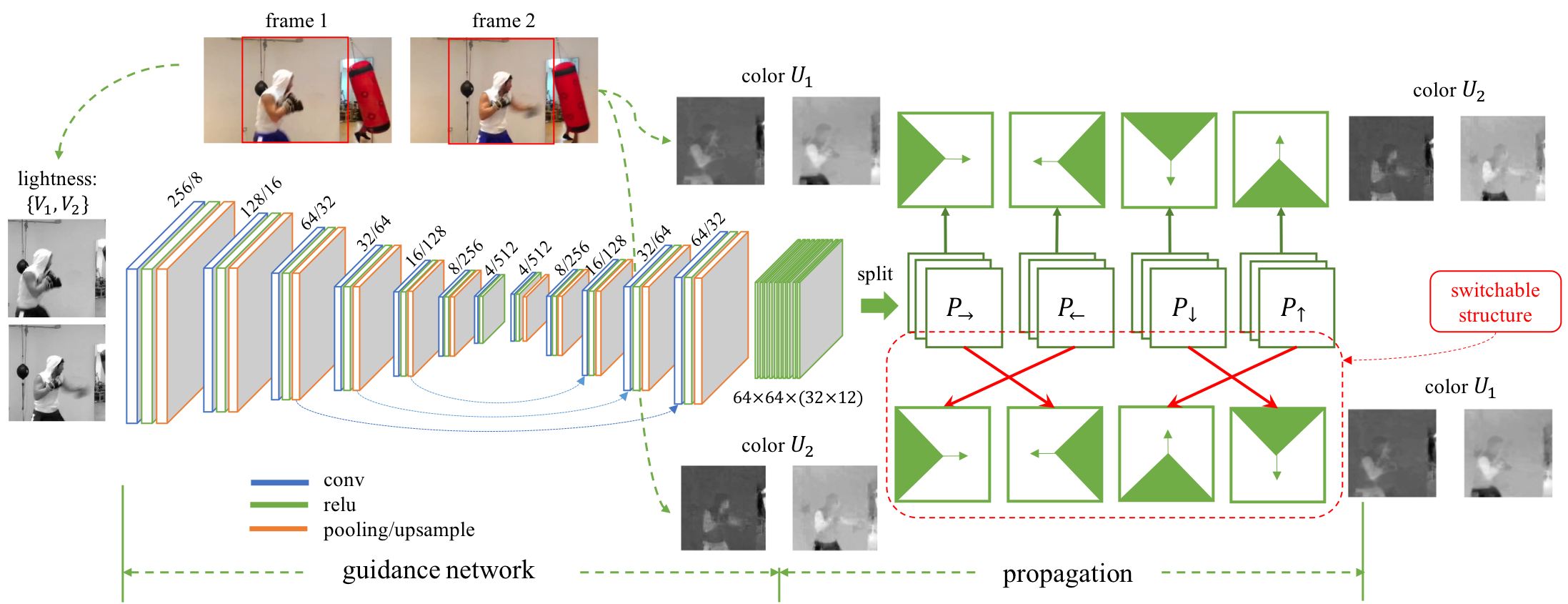}
	\caption{\footnotesize The architectures of a switchable TPN,
		which contains two propagation modules for bi-directional training.
		We specifically use a red-dashed box to denote the switchable structure
		In the reversed pair, the output channels $\left\lbrace P\right\rbrace $ are switched (red) for horizontal and vertical propagation.
	}
	\label{figure:stpn}\vspace{-5mm}
\end{figure*}

\vspace{-3mm}
\subsection{Learning  Pixel Transitions via the Basic TPN}\vspace{-2mm}
\label{sec:spn}
Directly learning the transformation matrix $G$ via a CNN is prohibited, since $G$ has a huge dimension (\eg, $n^2\times n^2$). Instead, inspired by the recent work~\cite{Liu17SPN}, we formulate the transformation as a diffusion process, and implement it efficiently by propagating information along each row and each column in an image linearly. 
Suppose we keep only $k=3$ nearest neighbors from a previous column (row) during the propagation, and we perform the propagation in $d=4$ directions, the total number of parameters to be estimated is significantly reduced from $n^2\times n^2$ to $n^2\times k \times d$ (see example of Fig.~\ref{figure:stpn}).

\vspace{-2mm}
\paragraph{\bf Linear transformation as a diffusion process.}
The diffusion process from frame $k$ to frame $t$ can be expressed with a partial differential equation (PDE) in the discrete form as:
\begin{equation}
\triangledown U = U_t-U_k= -LU_k = (A-D)U_k,
\label{eq:pde} 
\end{equation}
where $L=D-A$ is the Laplacian matrix, $D$ is the diagonal degree matrix and $A$ is the affinity matrix.
%
%
In our case, this represents the propagation of the property map $U$ over time.
%
\eqref{eq:pde} can be re-written as:
\begin{equation}
U_t =  \left( I-D+A\right)U_k = GU_k,
\label{eq:proof3} 
\end{equation}
where $G$ is the transformation matrix between the two states, as defined in~\eqref{eq:linear}, and $I$ is an identity matrix.
We can see that: (a) $G$ is entirely determined by the affinity matrix $A$, which quantifies the dense pairwise relationships between the pixels of the two property maps; (b) to model a diffusion process, $G$ should satisfy the property that $L$ is a standard Laplacian matrix.

%
%
%
%
%
%
\vspace{-2mm}
\paragraph{\bf Linear propagation network.}
With a propagation structure, the diffusion between frames can be implemented as a linear propagation along the rows or columns of an image.
Here we briefly show their equivalence.
Following~\cite{Liu17SPN}, we take the left-to-right spatial propagation operation as an example:
%
\begin{equation}
y_{i} = \left( I-d_i\right) x_i + w_i y_{i-1},\hspace{1em}i\in [2,n],
\label{eq:col-trans}
\end{equation}
where $x\in U_k$ and $y\in U_t$, and the $n\times1$ vectors $\left\lbrace x_i, y_i\right\rbrace $ represent the $i^{th}$ columns before and after propagation with an initial condition of $y_1=x_1$, and $w_i$ is the spatially variant $n\times n$ sub-matrix.
Here, $I$ is the identity matrix and $d_i$ is a diagonal matrix, whose $t^{th}$ element is the sum of all the elements of the $t^{th}$ row of $w_i$ as $d_i(t,t) = \sum_{j=1, j\neq t}^{n}w_i(j,t).$
Similar to \cite{Liu17SPN}, through (a) expanding its recursive term, and (b) concatenating all the rows/columns as a vectorized map, it is easy to prove that~\eqref{eq:col-trans} is equivalent to the global transformation $G$ between $U_k$ and $U_t$, 
where each entry is the multiplication of several spatially variant $w_i$ matrices. Please refer to the appendix for a complete proof.

%
%
%
%
Essentially, instead of predicting all the entries in $G$ as independent variables,
the propagation structure transfers the problem into learning each sub-matrix $w_i$ in~\eqref{eq:col-trans}, which significantly reduces the output dimensions.

\vspace{-2mm}
\paragraph{\bf Learning the sub-matrix $\left\lbrace w_i\right\rbrace $.} 
We adopt an independent deep CNN, namely the guidance network, to output all the sub-matrices $w_i$.
Note that the propagation in \eqref{eq:linear} is carried out for $d=4$ directions independently, as shown in Fig.~\ref{figure:stpn}.
For each direction, it takes a pair of images $\left\lbrace V_{k}, V_{t}\right\rbrace $ as its input, and outputs a feature map $P$ that has the same spatial size as $U$ (see Fig.~\ref{figure:stpn}).
Each pixel in the feature map $p_{i,j}$ contains all the values of the $j$th row in $w_i$, 
which describes a local relation between the adjacent columns, but results in a global connection in $G$ though the propagation structure.
Similar to~\cite{Liu17SPN}, we keep only $k=3$ nearest neighbors from the previous column, which results in $w_i$ being a tridiagonal matrix.
%
For each direction, we have a guidance network $g(\theta, V_{k}, V_{t})$, which produces an output $P$ of size $n\times n\times k$. Thus, a total of $n\times n \times k \times d$ parameters are used to implement the transformation matrix $G$. 
%
%
Such a structure significantly compresses the guidance network while still ensuring that the corresponding $G$ is a dense matrix that can describe global and dense pairwise relationships between a pair of frames.

\vspace{-2mm}
\subsection{Preserving Consistency via Switchable TPN}\vspace{-2mm}
\label{sec:cycle}
In this part, we show that there are two unique characteristics of propagation in the temporal domain, which do not exist for propagation in the spatial domain~\cite{Liu17SPN}.
First, temporal propagation is bi-directional for two frames, 
\ie, a network capable of transforming a frame $U_1$ into a frame $U_2$, should also be able to transform from $U_2$ to $U_1$, with a corresponding reversed ordering of inputs to the guidance network.
Second, during propagation, the overall ``style'' of the propagated property across the image should stay constant between frames,
\eg, during color propagation, the color saturation of all frames within a short video clip should be similar. 
We call this feature ``consistency property".
As shown below, we prove that enforcing the bi-directionality and the consistency is equivalent to ensure the transformation matrix $G$ to be orthogonal, which in turn can be easily implemented by equipping an ordinary temporal propagation network with a switchable structure.
%
%
\vspace{-2mm}
\paragraph{\bf Bi-directionality of TPN.}
Different from some high-level video-based tasks (\eg, action recognition) which may involve causal relations, low-level and mid-level properties in nearby video frames (\eg, color, HDR) often do not have a causal relationship. 
Hence, temporal propagation of these properties can often be switched in direction without breaking the process.
%
Therefore, given a diffusion model $G$ and a pair of frames $\left\lbrace U_1, U_2\right\rbrace$, we have a pair of equations:
\begin{eqnarray}
\vspace{-1mm}
U_2 = G_{1\rightarrow2}U_1, \quad
U_1 = G_{2\rightarrow1}U_2,
\label{eq:bi}
\vspace{-1mm}
\end{eqnarray}
\noindent where the arrow denotes the propagation direction.
%
%
The bi-directionality property implies that reversing the roles of the two frames as inputs by $\left\lbrace V_1, V_2\right\rbrace \rightarrow \left\lbrace V_2, V_1\right\rbrace$, and the corresponding supervision signals to the network corresponds to applying an inverse transformation matrix $G_{2\rightarrow1} = G_{1\rightarrow2}^{-1}$.
\vspace{-2mm}
\paragraph{\bf Style preservation in sequences.}
Style consistency refers to whether the generated frames can maintain similar chromatic properties or brightness when propagating color or HDR information, 
which is important for producing high-quality videos without the property vanishing over time.
In this work, we ensure such global temporal consistency by minimizing the difference in style loss of the propagated property for the two frames.
Style loss has been intensively used in style transfer~\cite{GatysEB15style}, but has not yet been used for regularizing temporal propagation.
In our work, we represent the style by the Gram matrix, which is proportional to the un-centered covariance of the property map.
The style loss is the squared Frobenius norm of the difference between the Gram matrices of the key-frame and the succeeding frame:
%
\begin{theorem}
	By regularizing the style loss we have the following optimization w.r.t. the guidance network:
	\vspace{-2mm}
	\begin{eqnarray}
	\min& \frac{1}{N}\parallel U_1^{\top}U_1 - U_2^{\top}U_2\parallel_F^{2} \label{eq:opt} \\
	s.t.& U_2 = GU_1 .
	\label{eq:style}
	\vspace{-2mm}
	\end{eqnarray}
	The optimal solution is reached when $G$ is orthogonal.
	\label{theo1}
\end{theorem}
\begin{proof}
Since the function~\eqref{eq:opt} is non-negative, the minimum is reached when $U_1^{\top}U_1=U_2^{\top}U_2$.
Combining it with~\eqref{eq:style} we have $G^{\top}G=I$.
\end{proof}
Given that $G$ is orthogonal, the $G_{2\rightarrow1}$ in \eqref{eq:bi} can be replaced by $G_{1\rightarrow2}^{\top}$, which equals to $G_{1\rightarrow2}^{-1}$.
Therefore, the bi-directionality propagation can be represented via a pair of transformation matrices that are \textit{transposed} w.r.t each other.
In the following part, we show how to enforce this property for the transformation matrix $G$ in the linear propagation network via a special network architecture.
Note that in our implementation, even though we use the channel-wise propagation described in Section.~\ref{sec:spn}, where the $U^{\top}U$ actually reduces to an uncentered variance, the conclusions of Theorem~\ref{theo1} still hold. 

\vspace{-2mm}
\paragraph{\bf A switchable propagation network.}
The linear transformation matrix $G$ has an important property:
since the propagation is directed, the transformation matrix $G$ is a triangular matrix.
Consider the two directions along the horizontal axis (\ie, $\rightarrow,  \leftarrow$) in Fig.~\ref{figure:stpn}. $G$ is an upper triangular matrix for a particular direction (\eg, $\rightarrow$), while it is lower triangular for the opposite one (\eg, $\leftarrow$). 
Suppose $P_\rightarrow$ and $P_\leftarrow$ are the output maps of the guidance network w.r.t. these two opposite directions.
This means that the transformation matrix, which is lower-triangular for propagation in the left-to-right direction, becomes upper-triangular for the opposite direction of propagation.
Since the upper-triangular matrix: (a) corresponds to propagating in the right-to-left direction, and (b) contains the same set of weight sub-matrices,
\emph{switching the CNN output channels w.r.t. the opposite directions $P_\rightarrow$ and $P_\leftarrow$ is equivalent to transposing the transformation matrix $G$ in the TPN.} This fact is exploited as a regularization structure (see the red bbox in Fig.~\ref{figure:stpn}), as well as an additional regularization loss term in \eqref{eq:loss} during training.

To summarize, the switchable structure of the TPN is derived from the two principles (\ie, the bi-directionality and the style consistency) for temporal propagation and
the fact that the matrix $G$ is triangular due to the specific form of propagation.
Note that \cite{Liu17SPN} did not address the triangulation of the matrix and thus were limited to propagation in the spatial domain only.
We show the switchable TPN (STPN) largely improve performance over the basic TPN, with no computational overhead at inference time.

\vspace{-3mm}
\section{Network Implementation}
\vspace{-2mm} 
We provide the network implementation details shared by color, HDR and segmentation mask propagation, which are proposed in this work. These settings can be potentially generalized to other properties of videos as well.
\vspace{-2mm}
\paragraph{\bf The basic TPN.}
The basic TPN contains two separate branches: (a) a deep CNN for the guidance network, which takes as input the provided information $\left\lbrace V_1, V_2\right\rbrace$ from a pair of frames, and outputs all elements ($P$) that constitute the state transformation matrix $G$, and (b) a linear propagation module that takes the property map of one frame $U_1$ and outputs $U_2$.
It also takes as input $\lbrace{P\rbrace}$ the propagation coefficients following the formulation 
of~\eqref{eq:col-trans}, where $\lbrace{P\rbrace}$ contains $kd$ channels ($k=3$ connections for each pixel per direction, and $d=4$ directions in total). 
$\left\lbrace V, U, P \right\rbrace$  have the same spatial size according to~\eqref{eq:col-trans}.
We use node-wise max-pooling~\cite{liu2016learning,Liu17SPN} to integrate the hidden layers and to obtain the final propagation result.
All submodules are differentiable and jointly trained using stochastic gradient descent (SGD), with the base learning rate of $10^{-5}$.

\vspace{-2mm}
\paragraph{\bf The switchable TPN.}
Fig.~\ref{figure:stpn} shows how the switchable structure of the TPN is exploited as an additional regularization loss term during training. 
For each pair $(U_1, U_2)$ of the training data, the first term in \eqref{eq:loss} shows the regular supervised loss between the network estimation $\hat{U}_2$ and the groundtruth $U_2$. In addition, as shown in Fig.~\ref{figure:stpn}(b), since we want to enforce the bi-directionality and the style consistency in the switchable TPN, the same network should be able to propagate from $U_2$ back to $U_1$ by simply switching the channels of the output of the guidance networks, \ie, 
switching the channels of $\left\lbrace P\rightarrow, P\leftarrow\right\rbrace $ and $\left\lbrace P\downarrow, P\uparrow\right\rbrace$ for propagating information in the opposite direction. This will form the second loss term in \eqref{eq:loss}, which serves as a regularization (weighted by $\lambda$) during the training.
%
%
%
%
%
%
%
We set $\lambda=0.1$ for all the experiments in this paper.
%
\begin{eqnarray}
	L(U_1, \hat{U}_1, U_2, \hat{U}_2) = 
	\left\| U_2(i)-\hat{U}_2(i)\right\|^2 + \lambda \left\| U_1(i)-\hat{U}_1(i)\right\|^2 .
	\vspace{-1mm}
	\label{eq:loss}
\end{eqnarray}
At inference time, the switchable TPN reduces to the basic TPN introduced in Section~\ref{sec:spn} and therefore does not have any extra computational expense.

\begin{figure}[t]
	\centering
	\includegraphics[width=0.99\linewidth]{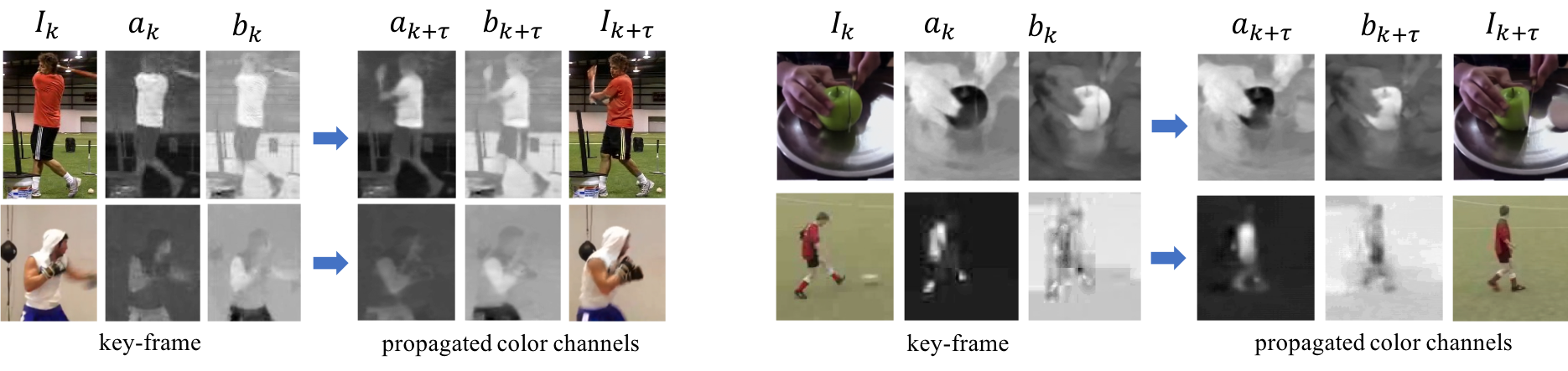}
	\caption{\footnotesize We show two groups of color transitions output through a basic TPN. For each column, the left size is key-frames with the ground truth color images provided, and the right side is new frames propagated from the left, correspondingly.
		$\left\lbrace a_k, b_k\right\rbrace $ and $\left\lbrace a_{k+\tau}, b_{k+\tau}\right\rbrace $ are the input and output of the TPN.
		All four examples show obvious appearance transitions caused by movement of objects.
		Our TPN successfully transfers those color maps in (a) to (b), and generates colored frames $I_{k+\tau}$. Zoom-in to see details.
	}
	\label{figure:transition}\vspace{-5mm}
\end{figure}

\begin{table}[t]
	\centering
	\footnotesize
	\caption{\footnotesize Run-times of different methods. We set $K=30$ for VPN color propagation~\cite{Jampani17VPN} to calculate its run-time. The last four columns are our methods.}
	\begin{tabular}{c|c|c|c||c|c|c|c}
		\hline
		 Method &  VPN~\cite{Jampani17VPN} (color) &  VPN~\cite{Jampani17VPN} (seg) &  HDRCNN\cite{Eilertsen17HDR} &  color  &  HDR &  SEG(t)  &  SEG(t+s) \\
		\hline
		 (ms)  &  730  & 750 &  365    &  15    &  25 & 17 & 84 \\ 
		\hline
	\end{tabular}\vspace{-5mm}
	\label{tab:runtime}
\end{table}

\section{Experimental Results}\vspace{-2mm}
\label{sec:exp}
In this section, we present our experimental results for propagating color channels, HDR images, and segmentation mask across videos.
We assume that the transition between frames is relatively small and can be modeled by image diffusion.
We note that propagating information across relatively longer temporal intervals may not satisfy the assumptions of a diffusion model, especially when new objects or scenes appear. 
Hence, for color and HDR propagation, instead of considering such complex scenarios, we set ``key-frames'' at regular fixed intervals for both tasks. 
That is, the ground truth color or HDR information is provided for every $K$ frames and propagated to all frames in between them.
This is a practical strategy for real-world applications. 
Note that for video segmentation mask propagation, we still follow the protocol of the DAVIS dataset~\cite{Perazzi2016} and only use the mask from the first frame.
%
%
\vspace{-1em}
\subsection{General Network Settings and Run-times}\vspace{-2mm}
We use a guidance network and a propagation module similar to~\cite{Liu17SPN}, with two cascaded propagation units.
%
%
For computational and memory efficiency, the propagation is implemented with a smaller resolution, where $U$ is downsampled from the original input space to a hidden layer before being fed into the propagation module.
The hidden layer is then bi-linearly upsampled to the original size of the image.
%
%
We adopt a symmetric U-net shaped, light-weight deep CNN with skip links for all tasks, but with different numbers of layers (see Fig.~\ref{figure:stpn} as an example for color propagation, and the appendix for detailed specification).
For both color and segmentation propagation, we first pre-train the model on synthesized frame pairs generated from the MS-COCO dataset~\cite{lin2014microsoft}.
Given an image, we augment it in two different ways via a similarity transform with uniformly sampled parameters from
$s\in[0.9,1.1], \theta\in[-15^{\circ},15^{\circ}]$ and $dx\in[-0.1, 0.1]\times b$, where $b=\min(H, W)$.
We also apply this data augmentation while training with patches from video sequences in the following stage to increase the variation of the training samples.
We present the run-times for different methods on an $512\times512$ image using a single TITAN X (Pascal) NVIDIA GPU (without cuDNN) in Table~\ref{tab:runtime}.

\begin{figure}[t]
	\centering
	\includegraphics[width=0.95\linewidth]{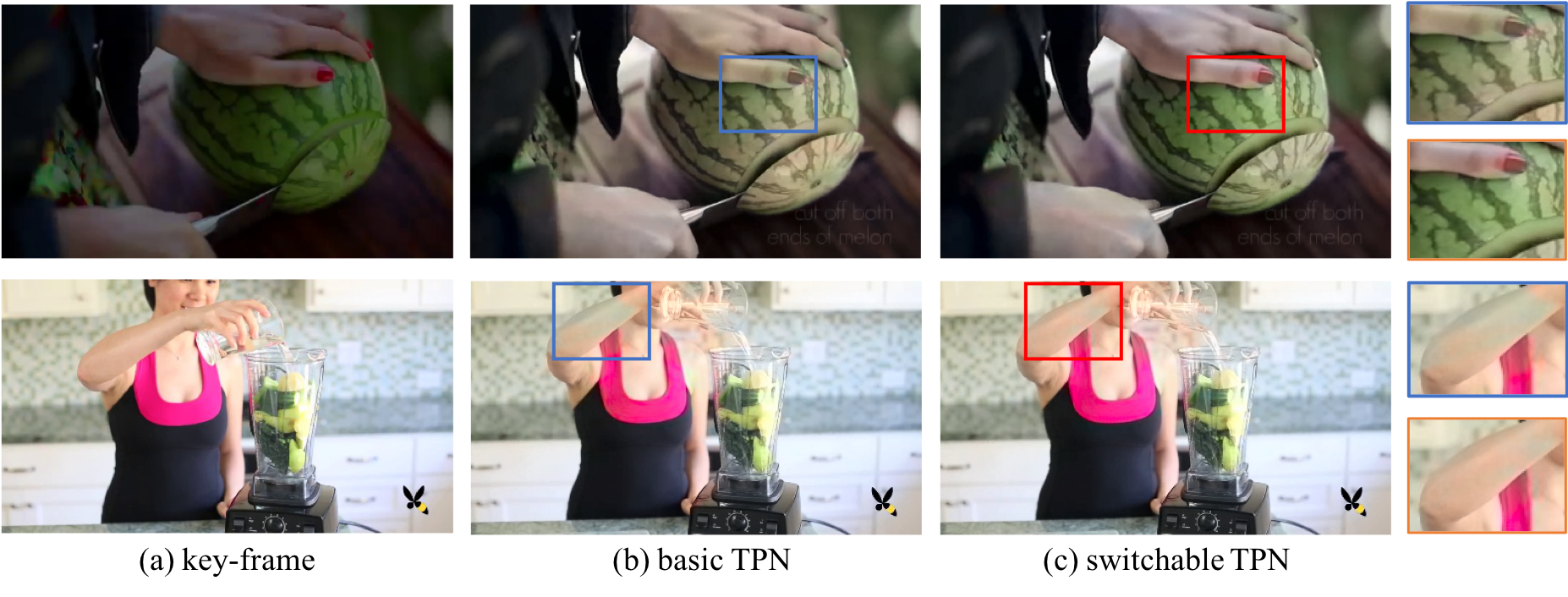}
	\caption{\footnotesize An example of color propagation from a key-frame (a) to a new frame with considerably large appearance transitions, using either (b) the basic TPN or (c) the switchable TPN. The closeups show the detailed comparison. Zoom-in to see details.
	}
	\vspace{-2em}
	\label{figure:stylepair}
\end{figure}

\vspace{-2mm}
\subsection{Color Propagation in Videos}\vspace{-2mm}
\label{sec:color}
Jampani et al. previously perform color propagation in videos using the video propagation network (VPN)~\cite{Jampani17VPN} through a bilateral network, and validate it on the DAVIS-16 dataset~\cite{Perazzi2016}.
However, since the DAVIS-16  training set is small (50 videos) and the sequences usually contain large transitions, we instead use the ACT dataset~\cite{Wang_Transformation}, which contains $7260$ training sequences with about 600K frames in total of various categories of actions.
All the sequences are short with small camera or scene transitions, and thus are more suitable for the proposed task.
We re-train and evaluate the VPN network on the ACT dataset for a fair comparison.
The original testing set contains $3974$ sequences with more than 300K frames.
For faster processing, we randomly select five videos from every action category in order to maintain the prior distribution of the original ACT dataset.  
We use one for testing and the remaining four for training.

We perform all computations in the \textit{CIE-Lab} color space.
After pretrained on the MS-COCO dataset, we fine-tune the models on the ACT dataset by randomly selecting two frames from a sequence and cropping both frames at the same spatial location to generate a single training sample.
Specifically, our TPN takes as input the concatenated \textit{ab} channels that are randomly cropped to $256\times 256$ from a key-frame.
The patch are then transformed to $64\times 64 \times 32$ via 2 convoluitional layers with $stride=2$ before being fed into the propagation module. 
After propagation, the output maps are upsampled 
and outputs a transformed \textit{ab} image map for the frames following the key-frame. 
The guidance CNN takes as input a pair of lightness images ($L$) for the two frames.
We optimize the Euclidean loss (in the \textit{ab} color space) between the ground truth and the propagated color channels generated by our network.
Note that for the switchable TPN, we have two losses with different weights according to~\eqref{eq:loss}.
During testing, we combine the estimated \textit{ab} channels with the given \textit{L} channel to generate a color RGB image.
All our evaluation metrics are computed in the \textit{RGB} color space.
More details are in the appendix.

\begin{table}[h]
	\centering
	\vspace{-10mm}
    \caption{\footnotesize RMSE and PSNR (in parentheses) for video color propagation on the ACT dataset for different key-frame interval $K$.
    We compared the results of VPN with $K=30$.}
    \footnotesize
    \begin{tabular}{c|cccc|cccc}
        \hline
         eval & \multicolumn{4}{c}{RMSE} & \multicolumn{4}{c}{PSNR} \\
        \hline
         Interval  &  $K=10$ &  $K=20$ &  $K=30$ &  $K=40$ &  $K=10$ &  $K=20$ &  $K=30$ &  $K=40$ \\
        \hline
         BTPNim+BTPNvd  &  4.43    &  5.46    &  6.04    &  6.44  &  36.65 &  35.22 &  34.46 &  33.96 \\ 
        \hline
         BTPNim+STPNvd &  4.00    &  5.00    &  5.58    &  6.01 &  37.63 &  36.09 &  35.26 &  34.70 \\ 
        \hline
         STPNim+STPNvd &  \textbf{3.98}    &  \textbf{4.97}    &  \textbf{5.55} &  \textbf{5.99} &  \textbf{37.64} &  \textbf{36.12} &  \textbf{35.29} &  \textbf{34.73} \\ \hline
 		 VPN (stage-1)~\cite{Jampani17VPN}  &    -    &  -    &  6.86 &  - &  -    &  -    &  32.86 &  - \\
        \hline
    \end{tabular}\vspace{-5mm}
	\label{tab:keyframe}
\end{table}

We compare the models with three combinations. We refer to the basic and switchable TPN networks as  BTPN and STPN, respectively. The methods that we compare include:
(a) BTPN on MS-COCO + BTPN on ACT, denoted by BTPNim+BTPNvd;
(b) BTPN on MS-COCO + STPN on ACT, denoted by BTPNim+STPNvd;
(c) STPN on MS-COCO + STPN on ACT, denoted by STPNim+STPNvd;
%
and evaluate different key-frames intervals, including $K=\left\lbrace10, 20, 30, 40\right\rbrace$.
The quantitative results for root mean square error (RMSE) and peak signal-to-noise ratio (PSNR) are presented in Table~\ref{tab:keyframe}.
Two trends can be inferred from the results.
First, the switchable TPN consistently outperforms the basic TPN and the VPN~\cite{Jampani17VPN}, and using the switchable TPN structure for both the pre-training and fine-tuning stages generates the best results.
Second, while the errors decrease drastically on reducing time intervals between adjacent key-frames, the colorized video maintains overall high-quality even when $K$ is set close to a common frame rate (\eg, 25 to 30 fps).
We also show in Fig.~\ref{figure:stylepair} (b) and (c) that the switchable structure significantly improves the qualitative results by preserving the saturation of color, especially when there are large transitions between the generated images and their corresponding key-frames.
The TPN also maintains good colorization for fairly long video sequences, which is evident from a comparison of the colorized video frames with the ground truth in Fig.~\ref{figure:colorVideo}. Over longer time intervals, the quality of the switchable TPN degrades much more gracefully than that of the basic TPN and the VPN~\cite{Jampani17VPN}.

\begin{figure}[t]
       \centering
       \includegraphics[width=0.99\linewidth]{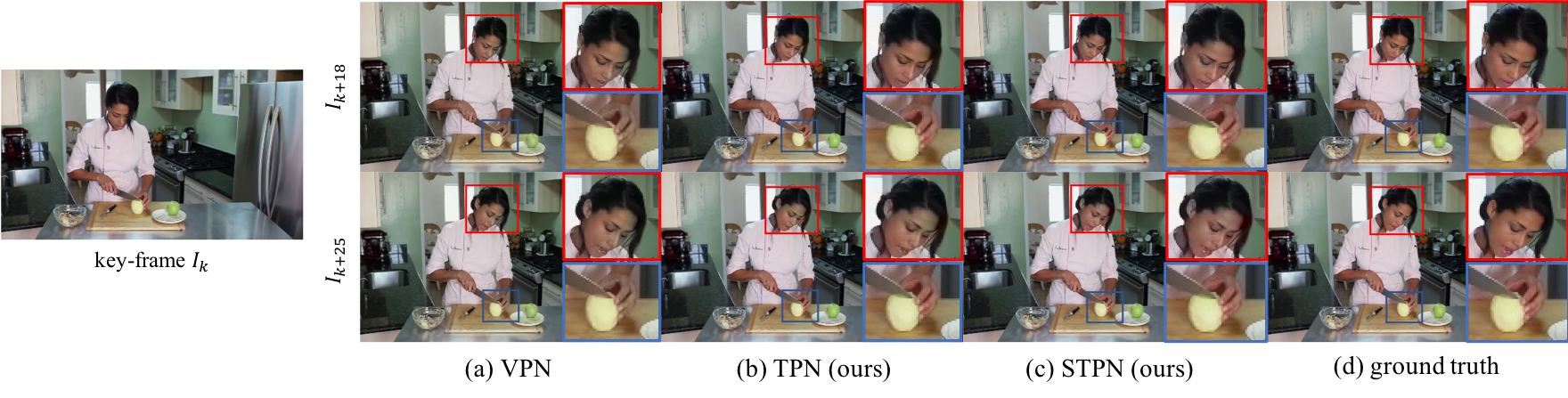}
       \caption{\footnotesize  Results using propagation of color from a key-frame to two proceeding frames (the $18^{th}$ and the $25^{th}$) at different time intervals with the basic/switchable TPN, and the VPN~\cite{Jampani17VPN} models. While both method in (b) and (c) produce better color propagation results, the STPN in (c) preserves better style consistency with longer ($k+25$ frames) time interval than the other two models. Zoom-in to see details.
       }
       \label{figure:colorVideo}\vspace{-7mm}
\end{figure}


\vspace{-2mm}
\subsection{HDR Propagation in Videos}
\label{sec:hdr}\vspace{-2mm}
%
Our method can utilize the ``key-frame'' to capture an HDR video, by shooting only one HDR image for every $K$ frames and processing the remaining $K-1$ frames with the TPN quickly and in real-time using guidance from their corresponding LDR images. 
The HDR key-frames can be generated by the existing cameras at a much lower frame-rate, by capturing multiple bracketed exposure shots and by combined them via computationally expensive image fusion techniques, which do not run in real-time.
Compared to the two alternative HDR capture methods, including (a) switching to the photo mode every $N$ frames via low-level camera control API or (b) alternating the exposure time for a few frames to obtain HDR~\cite{HDRVideo}, our method are more practical and low-cost.
%
%
We compare our method against the work of~\cite{Eilertsen17HDR}, which directly reconstructs the HDR frames given the corresponding LDR frames as inputs.
While this is not an apples-to-apples comparison because we also use an HDR key-frame as input, the work~\cite{Eilertsen17HDR} is the closest related state-of-the-art method to our approach for HDR reconstruction.
To our knowledge, no prior work exists on propagating HDR information in videos using deep learning and ours is the first work to address this problem.

We use similar network architecture as color propagation except that $U$ is transformed to $128\times 128 \times 16$ via one Conv layer to preserve more image details.
In addition, we also use a two-stage training procedure by first pre-training the network with randomly augmented pairs of patches created from a dataset of HDR images, and then fine-tuning on an HDR video dataset.
We collect the majority of the publicly available HDR image and video datasets listed in the appendix, and utilize all the HDR images and every 10-th frame of the HDR videos for training in the first stage~\cite{Eilertsen17HDR}. 
Except for the four videos (the same as~\cite{Eilertsen17HDR}) that we use for testing, we train our TPN with all the collected videos.
We evaluate our method on the four videos that~\cite{Eilertsen17HDR} used for testing and compare against their method. 

To deal with the long-tail, skewed distribution of pixel values in HDR images, similarly to~\cite{Eilertsen17HDR}, we use the logarithmic space for HDR training with $U=\log(H+\varepsilon)$, where $H$ denotes an HDR image and $\varepsilon$ is set to $0.01$.
Since the image irradiance values recorded in HDR images vary significantly across cameras, naively merging different datasets together often generates domain differences in the training samples.
To resolve this issue, before merging the datasets acquired by different cameras, we subtract from each input image the mean value of its corresponding dataset.
We use the same data augmentation as in~\cite{Eilertsen17HDR} of varying exposure values and camera curves~\cite{camera04} during training.

\begin{figure}[t]
       \centering
       \includegraphics[width=0.99\linewidth]{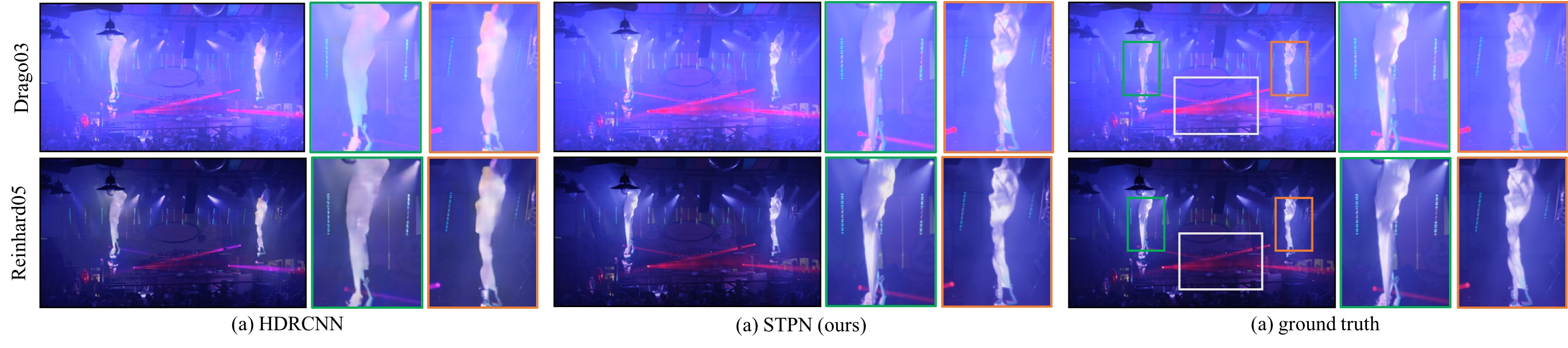}
       \caption{\footnotesize Results of HDR Video Propagation. We show one HDR frame ($\tau=19$ frames away from the key
       frame) reconstructed 
       with our switchable TPN (middle row). The top row shows the ground truth HDR, and the bottom row
       shows the output of HDRCNN~\cite{Eilertsen17HDR}. The HDR images are displayed with two popular 
       tone mapping algorithms, Drago03~\cite{Drago03} and Reinhard05~\cite{Reinhard05}. The insets show
       that the switchable TPN can effectively propagate the HDR information to new frames and preserve 
       the dynamic range of scene details.
   Zoom-in to see details.} 
       \label{fig:result_hdr}
\end{figure} 

During testing, we follow a similar approach to~\cite{Eilertsen17HDR}, which blends the inverse HDR image created from the input LDR image with the HDR image predicted by our TPN network to obtain the final output HDR image.
To do so, we assume a fixed camera response curve $g$ same as the one used in~\cite{Eilertsen17HDR}.
We also found a threshold $T$, with which the top $5\%$ of the pixels in the input key-frame are saturated. Given the camera response curve $g$ and the threshold $T$,  we then revert the input LDR to an HDR image (except for the saturated pixels). This inverted HDR image is blended linearly with the HDR prediction from our TPN as the final HDR output. 
More details are presented in the appendix.

\begin{table*}[t]
	\centering
	\footnotesize
	\vspace{-5mm}
	\caption{\footnotesize RMSE for video HDR propagation for the TPN output, the output with LDR blending, and for different intervals for the key-frames $K$. Reconstruction from single LDR~\cite{Eilertsen17HDR} is compared under the same experimental settings.}
	\footnotesize
	\begin{tabular}{c|cccc|cccc}
		\hline
		{ settings} & \multicolumn{4}{c|}{ HDR with blending} & \multicolumn{4}{c}{ HDR without blending} \\
		\hline
		 Interval  &  $K=10$ &  $K=20$ &  $K=30$ &  $K=40$ & $K=10$ &  $K=20$ &  $K=30$ &  $K=40$\\\hline
		 BTPNim+BTPNvd  &  $0.031$ &  $0.034$ &  $0.038$ &  $0.042$ & $0.119$ &  $0.160$ &  $0.216$ &  $0.244$\\
		 BTPNim+BTPNvd  &  $0.028$ &  $0.031$ &  $0.034$ &  $0.038$ & \bf{0.096} &  \bf{0.115} &  $0.146$ &  \bf{0.156}\\
		 BTPNim+BTPNvd  &  \bf{0.027} &  \bf{0.030} &  $0.034$ &  \bf{0.037} & $0.098$ &  $0.121$ &  \bf{0.142} &  $0.159$\\\hline
		 HDRCNN~\cite{Eilertsen17HDR}  & \multicolumn{4}{c|}{ 0.038} & \multicolumn{4}{c}{ 0.480} \\\hline
	\end{tabular}\vspace{-1em}
	\label{tab:hdr}
\end{table*}
%
%

%
%
We compare the RMSE of the generated HDR frames for different intervals between the key-frames, with or without the blending of LDR information with the HDR image generated by the TPN in Table~\ref{tab:hdr}.
Our results indicate that the switchable TPN can also significantly improve the results for HDR propagation compared to the basic TPN.
We also compare with the frame-wise reconstruction method~\cite{Eilertsen17HDR}, with and without the blending-based post-processing in Fig.~\ref{fig:result_hdr}. As shown, our TPN recovers HDR images with up to $K=30$ frames away from each key frame. The reconstructed HDR images preserve the same scene details as the ground truth, under different tone mapping algorithms. 
More results are presented in the appendix.
As noted earlier, since we have additional HDR key-frames as input, it is not an apples-to-apples comparison with single-image based HDR methods like~\cite{Eilertsen17HDR}. 
Nevertheless, the results in Fig.~\ref{fig:result_hdr} show the feasibility of using sparsely-sampled HDR key-frames to reconstruct HDR videos from LDR videos with the proposed TPN approach.

\vspace{-2mm}
\subsection{Segmentation Mask Propagation in Videos}\vspace{-2mm}
In addition, we conduct video segmentation on the DAVIS dataset~\cite{Perazzi2016} with the same settings as VPN~\cite{Jampani17VPN}, to show that the proposed method can also generated to semantic-level propagation in videos.
We note that maintaining style consistency does not apply to semantic segmentation.
For each frame to be predicted, we use the segmentation mask of the first frame as the only key-frame, while using the corresponding RGB images as the input to the guidance network.
We train two versions of the basic TPN network for this task:
(a) A basic TPN with the input/output resolution reduced to $256\times256$, the $U$ transformed to $64\times 64 \times 16$, in the same manner as the color propagation model. 
We used the same guidance network architecture as \cite{Liu17SPN}, while removing the last convolution unit to fit the dimensions of the propagation module.
This model, denoted as SEG(t) in Table~\ref{tab:runtime}, is much more efficient than the majority of the recent video segmentation methods~\cite{Jampani17VPN,Perazzi2016,Cae+17}.
(b) A more accurate model with an SPN~\cite{Liu17SPN} refinement applied to the output of the basic TPN, denoted as SEG(t+s). This model utilizes the same architecture as~\cite{Liu17SPN}, except that it replaces the loss with Sigmoid cross entropy for the per-pixel classification task.
Similar to color and HDR propagation, We pretrain (a) on the MS-COCO dataset and then finetune it on the DAVIS training set.
For the SPN model in (b), we first train it on the VOC image segmentation task as described in~\cite{Liu17SPN}. 
We treat each class in an image as binary mask in order to transfer the original model to a two-class classification model, while replacing the corresponding loss module.
We then finetune the SPN on the coarse masks from the DAVIS training set, which are produced by an intermediate model -- the pre-trained version of (a) from the MS-COCO dataset.
More details are introduced in the supplementary materiel.

\begin{table}[t]
	\centering
	\caption{\footnotesize 
	%
	Comparisons for video segmentation on the DAVIS dataset.}
	\footnotesize
	\vspace{2mm}
	\begin{tabular}{c|c||c|c|c|c||c|c}
		\hline
		\multicolumn{4}{c|}{J-mean} & \multicolumn{4}{c}{F-mean} \\
		\hline
		   VPN~\cite{Jampani17VPN} & OSVOS~\cite{Cae+17} &  SEG(t) &  SEG(t+s) & VPN~\cite{Jampani17VPN} & OSVOS~\cite{Cae+17} &  SEG(t) &  SEG(t+s)\\
		\hline
		   70.2 & 79.8  &  71.1 &  76.19  &  65.5 & 80.6  & 75.65 &  73.53\\ 
		\hline
	\end{tabular}
	\vspace{-2em}
	\label{tab:davis}
\end{table}

We compare our method to VPN~\cite{Jampani17VPN} and one recent state-of-the-art method~\cite{Cae+17}.
Both VPN and our method rely purely on the propagation module from the first frame and does not utilize any image segmentation pre-trained modules (in contrast with \cite{Cae+17}).
Similar to the other two tasks, both models significantly outperform VPN~\cite{Jampani17VPN} for video segmentation propagation (see Table~\ref{tab:davis}), while all running one order of magnitude faster (see Table~\ref{tab:runtime}).
The SEG(t+s) model performs comparatively to the OSVOS~\cite{Cae+17} method, which utilizes the pretrained image segmentation model and requires a much long inference time (7800 ms).
%
%
%

%

\section{Conclusions}
\label{sec:conclude}
We propose a switchable temporal propagation network for propagating image properties, such as colors, HDR luminance information, and segmentation masks through videos. 
To this end, we develop a method for temporal propagation that builds on image diffusion, but learns the pixel and high-level semantic affinities between images for a given a task based on a set of training data. 
We show that style regularization is needed, and that it can be enforced through bi-directional training, leading to our switchable temporal propagation network. 
We demonstrate the effectiveness of our approach on video colorization, LDR to HDR conversion and video segmentation.

The switchable TPN provides a general method for propagating information over time in videos. In the future, we will explore how to incorporate mid-level and high-level vision cues, such as detection, tracking, semantic/instance segmentation, for temporal propagation. 

\clearpage
\section*{Appendix}

\section*{Implementation Details}
We provide additional implementation details for all the video propagation tasks that we considered in the paper. All the source codes, models and dataset will be released to the public.
\paragraph{\bf Color propagation.}
For color propagation considered in the paper, we train both the basic and switchable TPN networks, for 10 epochs, on the COCO~\cite{lin2014microsoft} dataset, and then fine-tune them on the ACT~\cite{Wang_Transformation} dataset for 4 epochs. 
For both networks we keep the number of training epochs the same to ensure a fair comparison.
To pre-train our networks on the COCO image-based dataset, we sample two sets of similarity transformation parameters from $s\in[0.9,1.1], \theta\in[-15,15]$ and $dx\in[-0.1, 0.1]\times b$, where $b=\min(H, W)$ and use them to create two geometrically transformed versions of each training image.
Then we randomly crop $256\times 256$ sized patches from the same location of the two transformed images.
While fine-tuning the models on the ACT video-based dataset, we skip the step of applying geometric transformations to the frames and simply generate a pair of $256\times 256$ sized patches by randomly cropping from the same location of two different video frames.
Additionally, we randomly sample the pairs of video frames to use together, on-the-fly, during training.

\paragraph{\bf HDR propagation.}
The HDR images that we collect contain 664 images and 2477 sampled video frames in total.
We increase the images to about $100K$ pairs of $256\times 256$ sized patches by randomly augmenting each image with different parameters for the similarity transform and cropping patches (the cropping location is the same, while the similarity transform is different for a pair of samples).
The HDR video training set contains 24773 video frames (see Table~\ref{tab:hdrvideo}). We augment them and generate $100K$ pairs for offline training by randomly selecting two frames from a sequence as one training example.
We do not apply any geometric transformation to video frames in order to maintain the original motion information of the video, but simply crop them at the same location.

Following the notations of~\cite{Eilertsen17HDR}, our final HDR images (\textit{HDR with blending} in Table 3) are constructed by a simple pixel-wise blending method:
\begin{eqnarray}
&H(i) = f^{-1}\left( D(i)\right) ;\quad \{i|\hat{H}(i)<T\} \\\nonumber
&H(i) = \exp\left(y(i)\right) +\varepsilon;\quad \{i|\hat{H}(i)\geq T\} \\\nonumber
\label{eq:blending}
\end{eqnarray}
where $D$ is the LDR image, $\hat{H}(i)=\exp\left( y(i)\right) +\varepsilon$ is the estimated HDR image produced by the network without blending, $T$ is the threshold computed for the key-frames, which removes $5\%$ of the brightest pixels from it and is fixed for the all frames to which the information from the key frame is propagated, 
and $f(\cdot)$ is the camera curve proposed in~\cite{camera04}.
Note that the blending is not used during the training phase.
Instead of designing a soft blending mask as in~\cite{Eilertsen17HDR}, we directly replace the pixels with the revised LDR images according to the threshold $T$.

\paragraph{\bf Network architectures.}
Our guidance network has a symmetric U-net shaped structure, where both the input and output have the same size.
Since for the segmentation propagation in videos, the architecture follows that of the \cite{Liu17SPN} except that the top-most layer is removed in order to adapt to the lower resolution,
we only introduce the architectures w.r.t the applications of color and HDR propagation.
The down-sampling part of the network has seven consecutive conv+relu+max-pooling (with stride of 2) layers.
Starting from $8$, each layer has double the number of channels, resulting in $4\times 4\times 512$ feature maps at the bottleneck.
In order to use the information at different levels of image resolution, we add skipped-links by summing features maps of the same size from the corresponding down- and up-sampling layers.
The down- and up-sampling parts have symmetric configurations, except that for the up-sampling part the max-pooling layers are replaced by bilinear up-sampling layers. 
The last layer (the back-end of the propagation module) has 2 channels (\textit{ab} color channels) for color propagation, and 3 (\textit{RGB}) channels for HDR propagation.
The only difference between color and HDR propagation is the dimensionality of the propagation layer, as introduced in the paper.
Note that the network (both the CNN and the propagation part) is not restricted to be of the input image's size during the testing phase.
See Fig. \ref{figure:intro} in the paper as an example.

\vspace{-2mm}
\section*{Datasets and Additional Results}
We present more details of our combined HDR dataset and additional results.
\paragraph{\bf HDR datasets used in the paper.}
Since the availability of HDR images and videos is rare, we collected many publicly available HDR image and video datasets, which are listed in the Table~\ref{tab:hdr_image} and Table~\ref{tab:hdrvideo}.
\begin{table*}[h]
	\centering
	\vspace{-5mm}
	\caption{ HDR image datasets used in the paper.}
	\scriptsize
	\begin{tabular}{c|c|c|c}
		\hline
		\multicolumn{4}{c}{HDR image datasets} \\\hline
		name & source & training & testing \\\hline
		Deep HDR~\cite{Kalantari17HDR} & \url{http://cseweb.ucsd.edu/~viscomp/projects/SIG17HDR/} & 74 & 0 \\\hline
		ETHyma & \url{ftp://ftp.ivc.polytech.univ-nantes.fr/ETHyma/Images\textunderscore HDR/} & 11 & 0 \\\hline
		gward & \url{http://www.anyhere.com/gward/hdrenc/pages/originals.html} & 33 & 0 \\\hline
		hdrStanford & \url{http://scarlet.stanford.edu/~brian/hdr/hdrStanfordData.zip} & 88 & 0 \\\hline
		mpi-inf~\cite{cadik2011evaluation} & \url{http://resources.mpi-inf.mpg.de/hdr/gallery.html} & 13 & 0 \\\hline
		hdrlabs & \url{http://www.hdrlabs.com/sibl/archive.html} & 124 & 0 \\\hline
		Funt-HDR & \url{http://www.cs.sfu.ca/~colour/data/funt\_hdr/} & 224 & 0 \\\hline
		pfstools & \url{http://pfstools.sourceforge.net/hdr\textunderscore gallery.html} & 7 & 0 \\\hline
		hdr-eye~\cite{SundstedtVAE} & \url{https://mmspg.epfl.ch/hdr-eye} & 46 & 0 \\\hline
		Fairchild & \url{http://rit-mcsl.org/fairchild/HDRPS/HDRthumbs.html} & 104 & 0 \\\hline
		ESPL-LIVE~\cite{7944695} & \url{http://signal.ece.utexas.edu/~debarati/HDRDatabase.zip} & 150 & 0 \\\hline
	\end{tabular}
	\label{tab:hdr_image}\vspace{-4mm}
\end{table*}
\begin{table*}[h]
	\centering
	\vspace{-4mm}
	\caption{ HDR video datasets used in the paper.}
	\begin{tabular}{c|c|c|c}
		\hline
		\multicolumn{4}{c}{HDR video datasets (number of sequences/frames)} \\\hline
		name & source & training & testing \\\hline
		Boitard & \url{https://people.irisa.fr/Ronan.Boitard/} & 7/1915 & 0/0 \\\hline
		Stuttgart & \url{https://hdr-2014.hdm-stuttgart.de/} & 30/16208 & 3/1761 \\\hline
		MPI & \url{http://resources.mpi-inf.mpg.de/hdr/video/} & 2/1462 & 0/0 \\\hline
		DML-HDR & \url{http://dml.ece.ubc.ca/data/DML-HDR/} & 5/3006 & 0/0 \\\hline
		hdrv~\cite{Eilertsen17HDR} & \url{http://hdrv.org/} & 9/2182 & 1/401 \\\hline 
	\end{tabular}
	\label{tab:hdrvideo}\vspace{-4mm}
\end{table*}

\paragraph{\bf A new self-collected HDR video dataset.}
In addition, we also develop a new HDR video dataset containing diverse indoor/outdoor scenes. 
The dataset contains 17 videos with about $13K$ frames.
We test the HDRCNN~\cite{Eilertsen17HDR} as well as our switchable TPN on all the 17 videos using the same settings and evaluation criteria as in the paper.
Note that no video is used for training or finetuning the models in this paper.
We will provide more videos which supply as training data for various HDR related tasks.
All the data will be made available to the public.
\begin{table*}[h]
		\centering
		\vspace{-10mm}
	\caption{ RMSE for video HDR propagation for the TPN output as well as the output with LDR blending on our self-collected dataset. We fix the interval of key-frame $K=20$. Reconstruction from single LDR~\cite{Eilertsen17HDR} is compared under the same experimental settings.}
	\begin{tabular}{c|c|c}
		\hline
		methods & HDR with blending & HDR without blending \\\hline
		HDRCNN~\cite{Debevec97HDR} & 0.013 & 0.403 \\\hline
		switchable TPN $K=20$ & 0.008 & 0.039 \\\hline
	\end{tabular}	\vspace{-10mm}
\end{table*}

\paragraph{\bf Online vs. offline propagation of color frames.}
Different from the application of HDR frame propagation which is designed as a prototype for HDR video camera, color propagation can be used for simplifying video colorization and editing.
Therefore, the processing of video colorization can be offline, i.e., any frame to be colorized can be produced by utilizing the information from two key-frames -- the one that is preceding or following it, instead of a single one before it.
In this part, we provide the performance of a simply way to generate offline propagation results:
the propagation result of each frame is produced by selecting the nearest key-frame out of the two key-frames.
Such method provides visually smoother qualitative results for video color propagation when videos can be processed offline.
\begin{table}[h]
	\centering
	\vspace{-5mm}
	\caption{ RMSE and PSNR (in parentheses) for \textbf{offline} video color propagation on the ACT dataset for different key-frame spacing $K$.
	}
	\begin{tabular}{c|cccc|cccc}
		\hline
		eval & \multicolumn{4}{c}{RMSE} & \multicolumn{4}{c}{PSNR} \\
		\hline
		 Interval  &  $K=10$ &  $K=20$ &  $K=30$ &  $K=40$ &  $K=10$ &  $K=20$ &  $K=30$ &  $K=40$ \\
		\hline
		 BTPNim+BTPNvd   &  3.48    &  4.26    &  4.77    &  5.11  &  38.50 &  37.06 &  36.29 &  35.76 \\ 
		\hline
		 BTPNim +STPNvd&  3.03   &  3.80    &  4.29    &  4.63  &  39.78 &  38.16 &  37.31 &  36.71 \\ 
		\hline
		 STPNim +STPNvd&  \textbf{3.02}    &  \textbf{3.79}    &  \textbf{4.27} &  \textbf{4.60}  &  \textbf{39.78} &  \textbf{38.16} &  \textbf{37.31} &  \textbf{36.72} \\ \hline
	\end{tabular}
	\label{tab:keyframe2}	\vspace{-5mm}
\end{table}

\clearpage
\paragraph{\bf Color transitions.}
We show more examples of color propagation between one key-frame, and a random proceeding frame, which has a significantly different appearance from the key-frame ($\tau\in\left[ 1,29\right] $) in Fig.~\ref{figure:transition2}.
\begin{figure*}[h]
	\centering
	\vspace{-7mm}
	\includegraphics[width=0.82\linewidth]{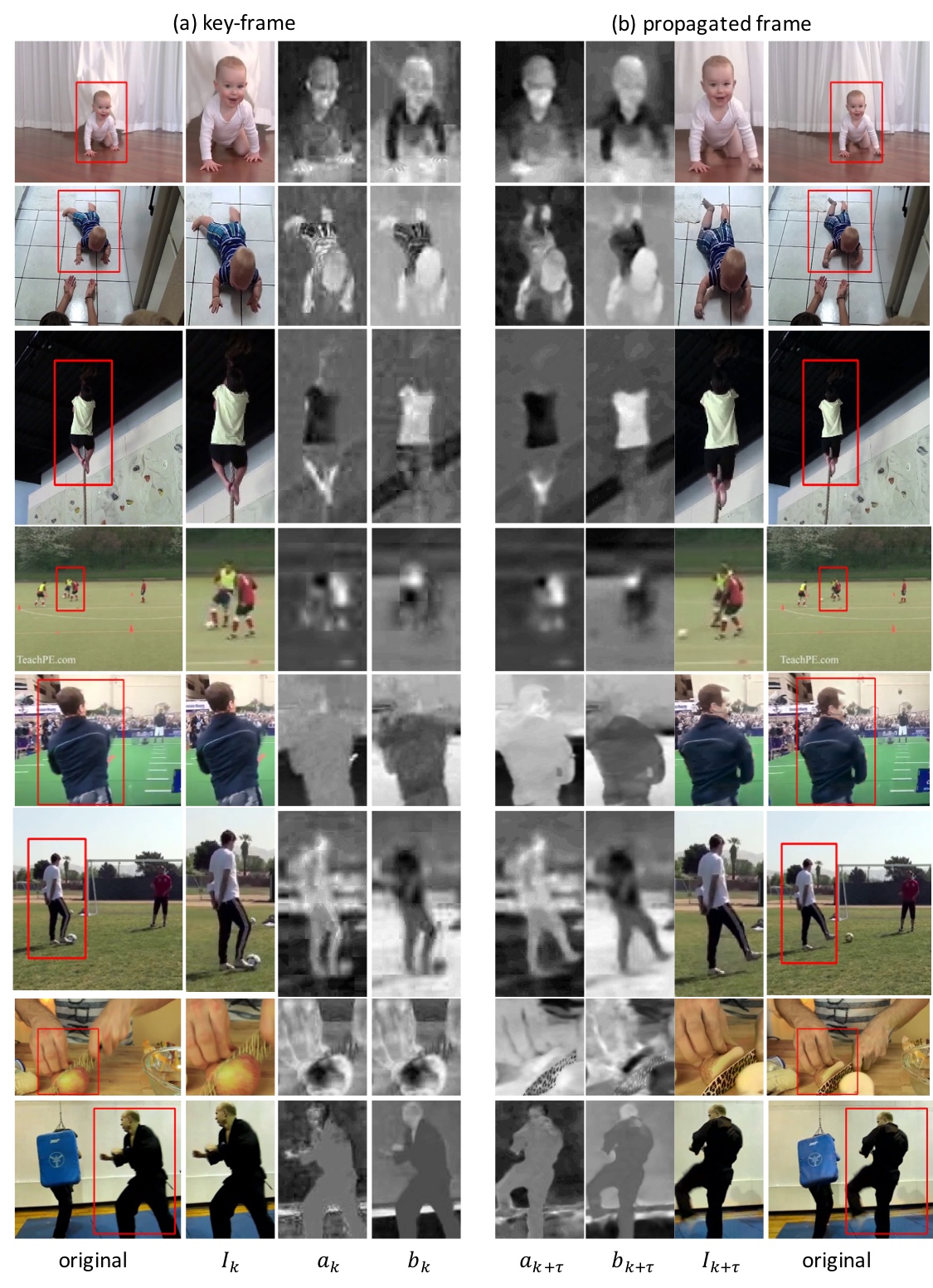}
	\caption{\footnotesize We show results of color propagation by the basic TPN where
		(a) shows key-frames with the ground truth color images provided, and (b) shows a new color frame produced by propagating the color from (a), correspondingly.
		$\left\lbrace a_k, b_k\right\rbrace $ and $\left\lbrace a_{k+\tau}, b_{k+\tau}\right\rbrace $ are the inputs and outputs of the TPN, respectively.
		All examples show obvious appearance transitions caused by the movement of objects between the frames.
		Our TPN successfully transfers the color maps in (a) to (b), and produces colored frames $I_{k+\tau}$.
	}
	\label{figure:transition2}\vspace{-9mm}
\end{figure*}

\clearpage
\paragraph{\bf Basic TPN vs. switchable TPN for color propagation.}
We show more examples of comparisons between the basic and switchable TPN architectures in Fig.~\ref{figure:stylepair2}.
\begin{figure*}[h]
	\centering
	\vspace{-7mm}
	\includegraphics[width=0.82\linewidth]{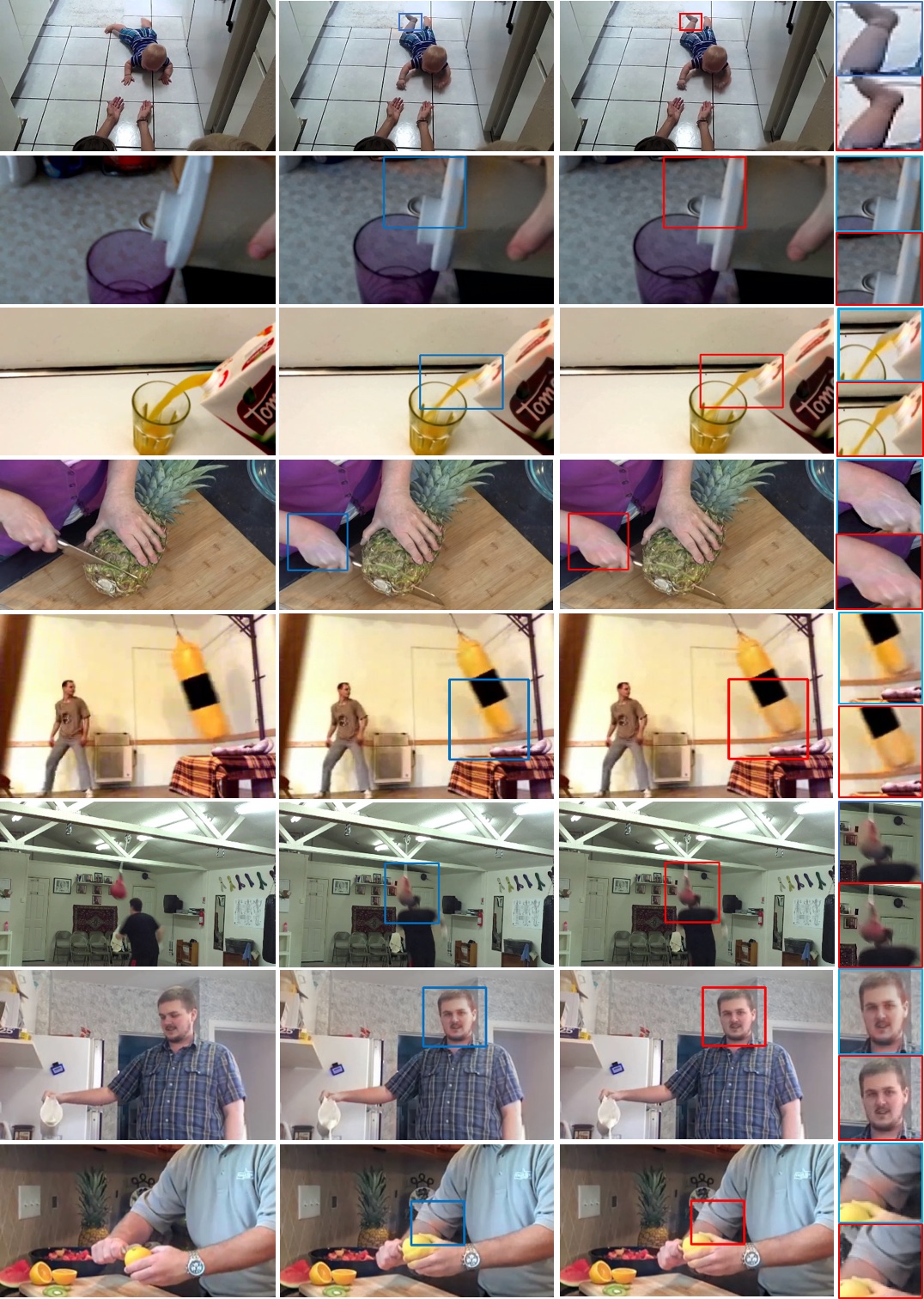}\vspace{3mm}
	\caption{We show the propagation of color from a key-frame (first column) to a new frame, including the result produced by (second column) basic TPN and (third column) switchable TPN. 
		Details are shown in the pairs of close-up images on the right and can be observed by zooming in.
	}
	\label{figure:stylepair2}\vspace{-9mm}
\end{figure*}

\clearpage
\paragraph{\bf Switchable TPN vs. VPN~\cite{Jampani17VPN} for color propagation.}
We show more examples of the comparison between the switchable TPN and the video propagation network (VPN)~\cite{Jampani17VPN} in Fig.~\ref{figure:STPNvsVPN1} and Fig.~\ref{figure:STPNvsVPN2}.
\vspace{-2mm}
\begin{figure}[h]
	\centering
	\includegraphics[width=0.99\linewidth]{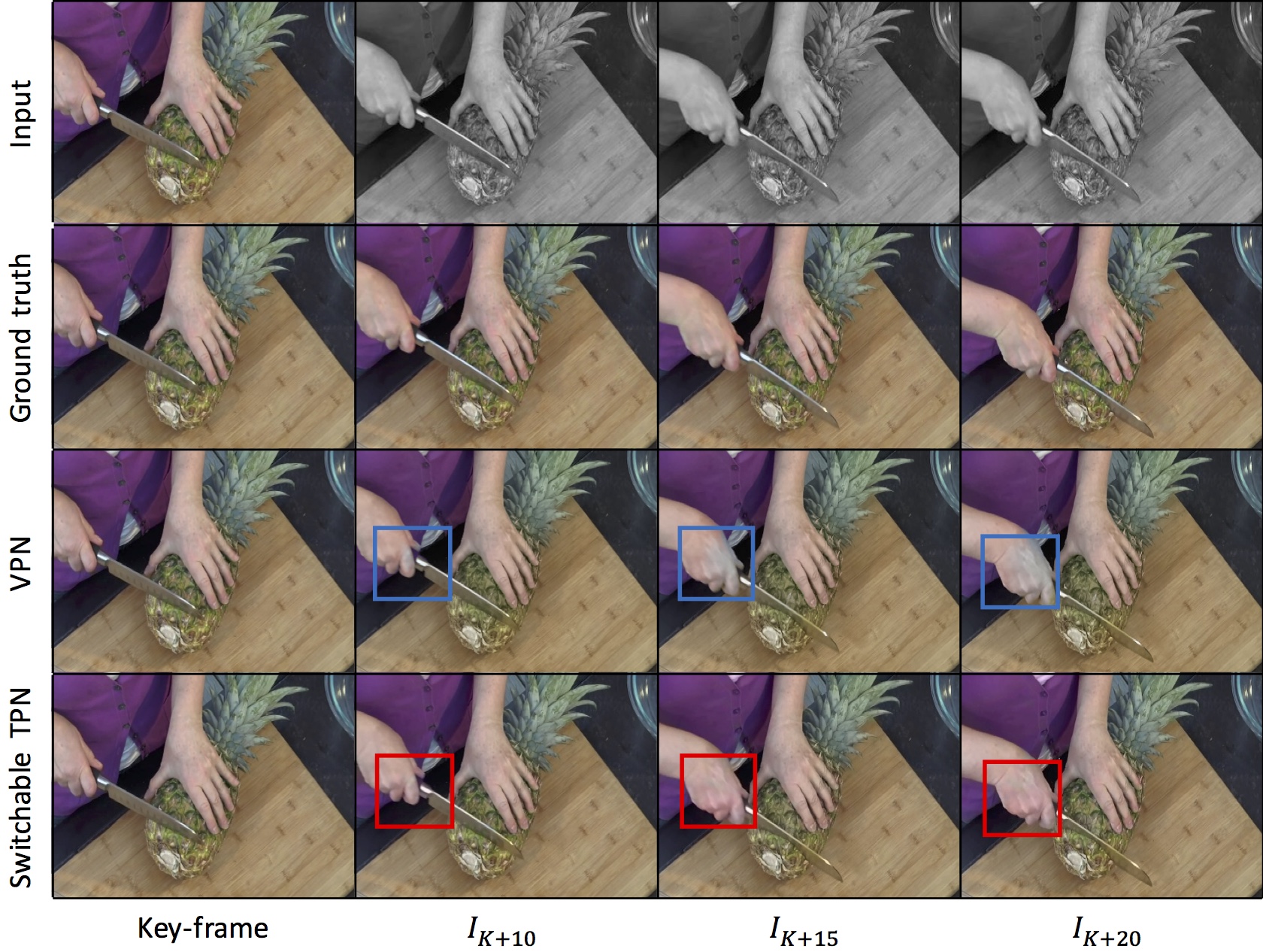}
	\caption{ \footnotesize We show the propagation of color from a key-frame (first column) to several proceeding frames (columns 2-4), including the result produced by VPN~\cite{Jampani17VPN} in the third row and our switchable TPN (fourth row). The first row shows the colored key-frame and the subsequent grayscale frames that are input to the network and the second row contains the ground truth colored frames. Details are highlighted in the red and blue boxes and can best viewed by zooming in.
	}
	\label{figure:STPNvsVPN1}
\end{figure}

\vspace{-2mm}
\begin{figure}[h]
	\centering
	\includegraphics[width=0.99\linewidth]{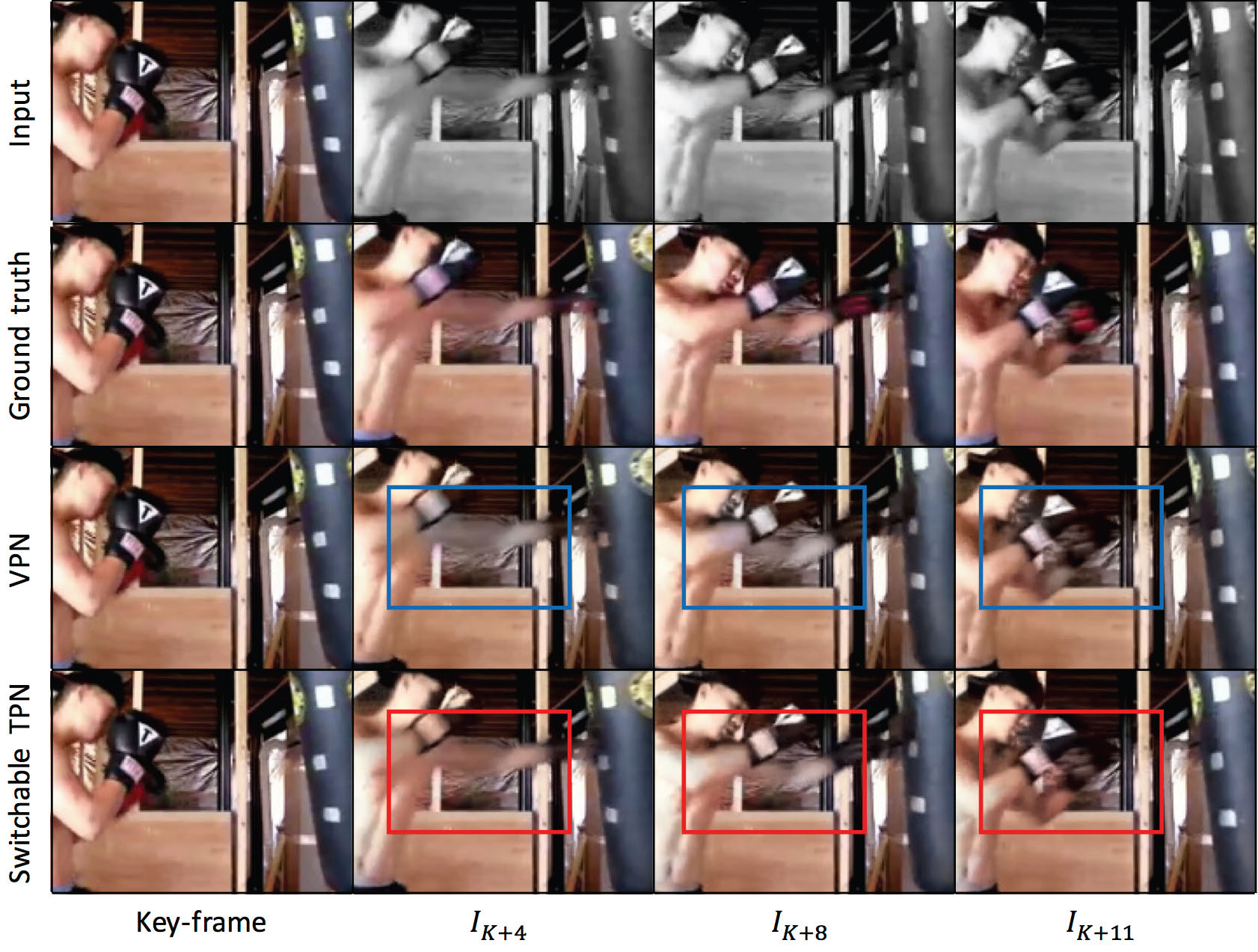}
	\caption{ \footnotesize We show the propagation of color from a key-frame (first column) to several proceeding frames (columns 2-4), including the result produced by VPN~\cite{Jampani17VPN} in the third row and our switchable TPN (fourth row). The first row shows the colored key-frame and the subsequent grayscale frames that are input to the network and the second row contains the ground truth colored frames. Details are highlighted in the red and blue boxes and can best viewed by zooming in.
	}
	\label{figure:STPNvsVPN2}
\end{figure}

\clearpage
\paragraph{\bf Switchable TPN vs. HDRCNN~\cite{Eilertsen17HDR} for HDR propagation.}
We show more examples of the comparison between the single LDR image-based HDR reconstruction approach~\cite{Eilertsen17HDR} and our method.
We keep the key-frame interval $K=20$ and show our results with relatively large time intervals $\tau$, where the frames to which the HDR information is propagated have obvious transitions from the key-frame.
Specifically, we use relatively low exposure and fix gamma to 2.2 to produce tone-mapped LDR images, in order to show the details recovered in the saturated regions.
Note that other details in the images can also be viewed by adjusting the exposure value to produce different tone-mapped LDR images (\eg, the video corresponding to Fig.~\ref{figure:hdr2}, which is tone-mapped using the \textit{localtonemap} function in MATLAB to show rich details, is embedded in the video file, which we supply with the appendix).
\begin{figure}[h]
	\centering
	\includegraphics[width=0.92\linewidth]{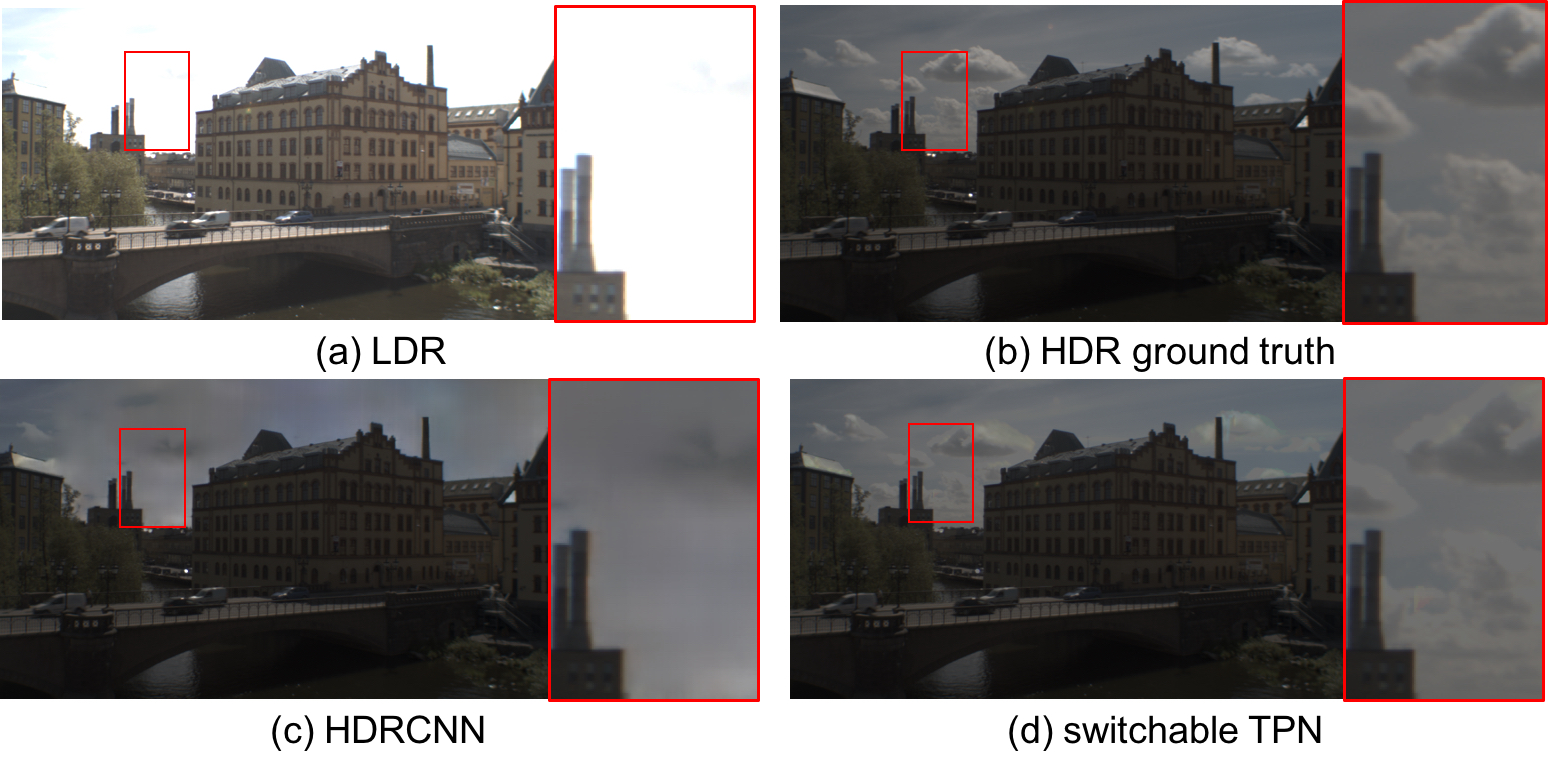}
	\caption{ \footnotesize We show the propagation of HDR from a key-frame (a) to a proceeding frame with $\tau=10$.
		This video has large and rapid camera transitions over time.
		While the HDRCNN~\cite{Eilertsen17HDR} in (c) loses many details, the switchable TPN in (d) can successfully track the details in the saturated region (see (a)).
	}
	\label{figure:hdr1}
\end{figure}
\vspace{-10mm}
\begin{figure}[h]
	\centering
	\includegraphics[width=0.92\linewidth]{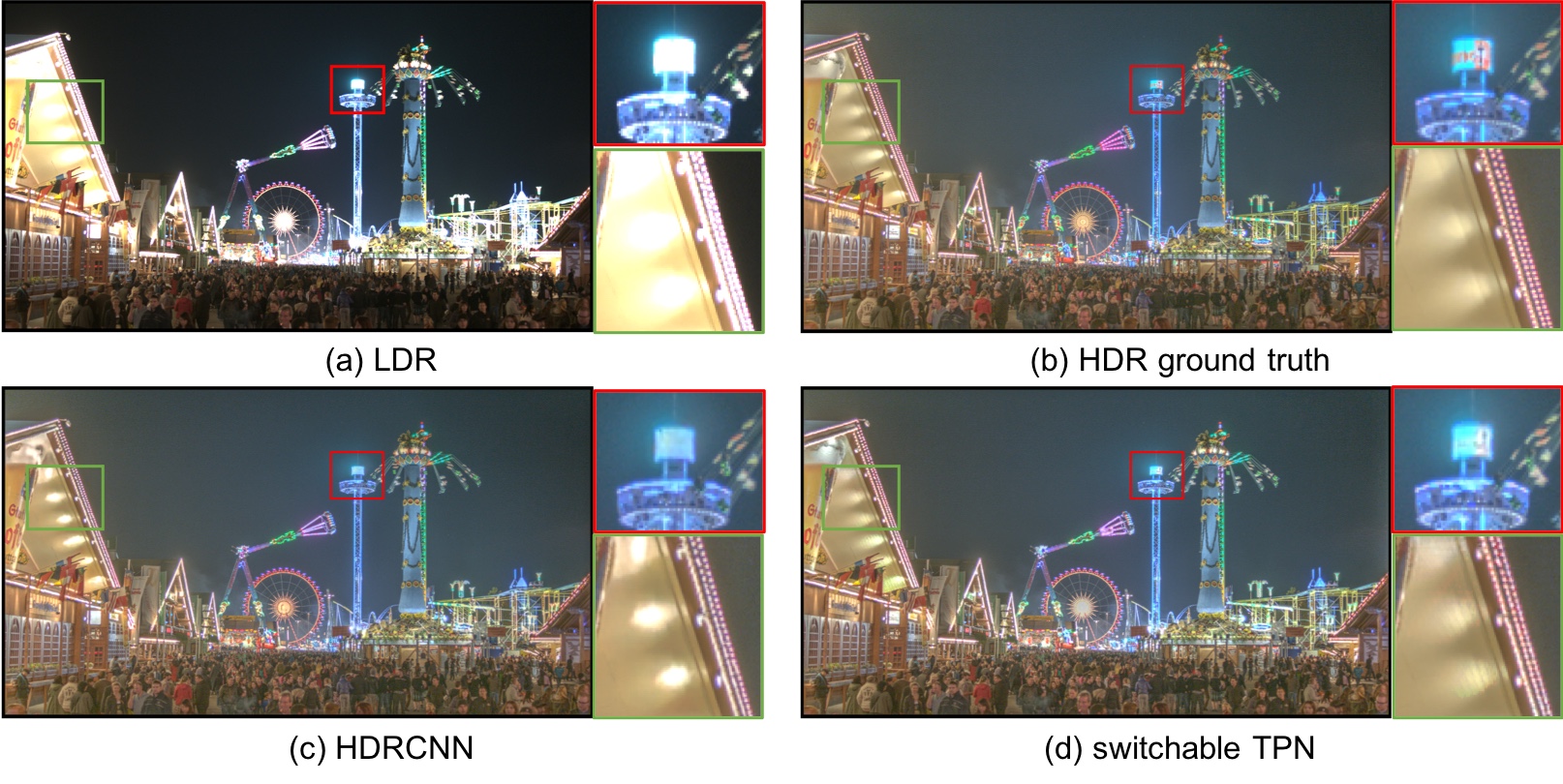}
	\caption{ \footnotesize We show the propagation of HDR from a key-frame (a) to a proceeding frame with $\tau=10$.
		This video contains rapid variations in the lamp regions (red), and saturated still regions (green).
		The switchable TPN in (d) can recover the details in the lamp region even with rapid motions. Also it is helpful to maintain the HDR information in the still region for the whole sequence.
	}
	\label{figure:hdr2}
\end{figure}

\clearpage

\bibliographystyle{plain}


\begin{thebibliography}{10}

\bibitem{gadde2017semantic}
Gadde, R., Jampani, V., Gehler, P.V.:
\newblock Semantic video {CNNs} through representation warping.
\newblock In: Proceedings of IEEE International Conference on Computer Vision
  (ICCV). (2017)

\bibitem{he2013guided}
He, K., Sun, J., Tang, X.:
\newblock Guided image filtering.
\newblock IEEE Transactions on Pattern Analysis and Machine Intelligence
  (TPAMI) \textbf{35}(6) (2013)  1397--1409

\bibitem{levin2004colorization}
Levin, A., Lischinski, D., Weiss, Y.:
\newblock Colorization using optimization.
\newblock In: ACM Transactions on Graphics (TOG). (2004)

\bibitem{Jampani17VPN}
Jampani, V., Gadde, R., Gehler, P.:
\newblock Video propagation networks.
\newblock In: Proceedings of IEEE Conference on Computer Vision and Pattern
  Recognition (CVPR). (2017)

\bibitem{Eilertsen17HDR}
Eilertsen, G., Kronander, J., Denes, G., Mantiuk, R., Unger, J.:
\newblock {HDR} image reconstruction from a single exposure using deep {CNNs}.
\newblock In: ACM Transactions on Graphics (SIGGRAPH Asia). (2017)

\bibitem{levin2008closed}
Levin, A., Lischinski, D., Weiss, Y.:
\newblock A closed-form solution to natural image matting.
\newblock IEEE Transactions on Pattern Analysis and Machine Intelligence
  (TPAMI) \textbf{30}(2) (2008)

\bibitem{Liu17SPN}
Liu, S., Mello, S.D., Gu, J., Zhong, G., Yang, M., Kautz, J.:
\newblock Learning affinity via spatial propagation networks.
\newblock In: Neural Information Processing Systems (NIPS). (2017)

\bibitem{jampanicvpr2016}
Jampani, V., Kiefel, M., Gehler, P.V.:
\newblock Learning sparse high dimensional filters: Image filtering, dense
  {CRFs} and bilateral neural networks.
\newblock In: Proceedings of IEEE Conference on Computer Vision and Pattern
  Recognition (CVPR). (2016)

\bibitem{zhang2016colorful}
Zhang, R., Isola, P., Efros, A.A.:
\newblock Colorful image colorization.
\newblock In: Proceedings of European Conference on Computer Vision (ECCV).
  (2016)

\bibitem{zhang2017real}
Zhang, R., Zhu, J.Y., Isola, P., Geng, X., Lin, A.S., Yu, T., Efros, A.A.:
\newblock Real-time user-guided image colorization with learned deep priors.
\newblock In: ACM Transactions on Graphics (SIGGRAPH). (2017)

\bibitem{Debevec97HDR}
Debevec, P., Malik, J.:
\newblock Recovering high dynamic range radiance maps from photographs.
\newblock In: ACM Transactions on Graphics (SIGGRAPH). (1997)

\bibitem{Reinhard10HDR}
Reinhard, E., Heidrich, W., Debevec, P., Pattanaik, S., Ward, G., Myszkowski,
  K.:
\newblock High Dynamic Range Imaging: Acquisition, Display, and Image-based
  Lighting.
\newblock Morgan Kaufmann (2010)

\bibitem{Hu13HDR}
Hu, J., Gallo, O., Pulli, K., Sun, X.:
\newblock {HDR} deghosting: How to deal with saturation?
\newblock In: Proceedings of IEEE Conference on Computer Vision and Pattern
  Recognition (CVPR). (2013)

\bibitem{Oh15HDR}
Oh, T., Lee, J., Tai, Y., Kweon, I.:
\newblock Robust high dynamic range imaging by rank minimization.
\newblock IEEE Transactions on Pattern Analysis and Machine Intelligence
  (TPAMI) \textbf{37}(6) (2015)  1219--1232

\bibitem{Gallo15HDR}
Gallo, O., Troccoli, A., Hu, J., Pulli, K., Kautz, J.:
\newblock Locally non-right registration for mobile {HDR} photography.
\newblock In: Proceedings of IEEE Conference on Computer Vision and Pattern
  Recognition (CVPR). (2015)

\bibitem{Kang03HDR}
Kang, S., Uyttendaele, M., Winder, S., Szeliski, R.:
\newblock High dynamic range video.
\newblock In: ACM Transactions on Graphics (SIGGRAPH). (2003)

\bibitem{Kalantari13HDR}
Kalantari, N., Shechtman, E., Barnes, C., Darabi, S., Goldman, D., Sen, P.:
\newblock Patch-based high dynamic range video.
\newblock In: ACM Transactions on Graphics (SIGGRAPH). (2013)

\bibitem{Tocci11HDR}
Tocci, M., Kiser, C., Tocci, N., Sen, P.:
\newblock A versatile {HDR} video production system.
\newblock In: ACM Transactions on Graphics (SIGGRAPH). (2011)

\bibitem{Nayar00HDR}
Nayar, S., Mitsunaga, T.:
\newblock High dynamic range imaging: Spatially varying pixel exposure.
\newblock In: Proceedings of IEEE Conference on Computer Vision and Pattern
  Recognition (CVPR). (2000)

\bibitem{Gu10CRSP}
Gu, J., Hitomi, Y., Mitsunaga, T., Nayar, S.:
\newblock Coded rolling shutter photography: Flexible space-time sampling.
\newblock In: Proceedings of IEEE International Conference on Computational
  Photography (ICCP). (2010)

\bibitem{Kalantari17HDR}
Kalantari, N., Ramamoorthi, R.:
\newblock Deep high dynamic range imaging of dynamic scenes.
\newblock In: ACM Transactions on Graphics (SIGGRAPH). (2017)

\bibitem{Zhang17HDR}
Zhang, J., Lalonde, J.:
\newblock Learning high dynamic range from outdoor panoramas.
\newblock In: Proceedings of IEEE International Conference on Computer Vision
  (ICCV). (2017)

\bibitem{GatysEB15style}
Gatys, L.A., Ecker, A.S., Bethge, M.:
\newblock A neural algorithm of artistic style.
\newblock CoRR \textbf{abs/1508.06576} (2015)

\bibitem{liu2016learning}
Liu, S., Pan, J., Yang, M.H.:
\newblock Learning recursive filters for low-level vision via a hybrid neural
  network.
\newblock In: Proceedings of European Conference on Computer Vision (ECCV).
  (2016)

\bibitem{Perazzi2016}
Perazzi, F., Pont-Tuset, J., McWilliams, B., {Van Gool}, L., Gross, M.,
  Sorkine-Hornung, A.:
\newblock A benchmark dataset and evaluation methodology for video object
  segmentation.
\newblock In: Proceedings of IEEE Conference on Computer Vision and Pattern
  Recognition (CVPR). (2016)

\bibitem{lin2014microsoft}
Lin, T.Y., Maire, M., Belongie, S., Hays, J., Perona, P., Ramanan, D.,
  Doll{\'a}r, P., Zitnick, C.L.:
\newblock Microsoft {COCO}: Common objects in context.
\newblock In: Proceedings of European Conference on Computer Vision (ECCV).
  (2014)  740--755

\bibitem{Wang_Transformation}
Wang, X., Farhadi, A., Gupta, A.:
\newblock Actions \protect{$\sim$} transformations.
\newblock In: Proceedings of IEEE Conference on Computer Vision and Pattern
  Recognition (CVPR). (2016)

\bibitem{HDRVideo}
Kalantari, N.K., Shechtman, E., Barnes, C., Darabi, S., Goldman, D.B., Sen, P.:
\newblock Patch-based high dynamic range video.
\newblock Volume~32. (2013)

\bibitem{camera04}
Grossberg, M.D., Nayar, S.K.:
\newblock Modeling the space of camera response functions.
\newblock IEEE Transactions on Pattern Analysis and Machine Intelligence
  (TPAMI) \textbf{26}(10) (2004)  1272--1282

\bibitem{Drago03}
Drago, F., Myszkowski, K., Annen, T., Chiba, N.:
\newblock Adaptive logarithmic mapping for displaying high contrast scenes.
\newblock Compute Graphics Forum \textbf{22}(3) (2003)  419--426

\bibitem{Reinhard05}
Reinhard, E., Devlin, K.:
\newblock Dynamic range reduction inspired by photoreceptor physiology.
\newblock IEEE Transactions on Visualization and Computer Graphics
  \textbf{11}(1) (2005)  13--24

\bibitem{Cae+17}
Caelles, S., Maninis, K.K., Pont-Tuset, J., Leal-Taix\'e, L., Cremers, D., {Van
  Gool}, L.:
\newblock One-shot video object segmentation.
\newblock In: Proceedings of IEEE Conference on Computer Vision and Pattern
  Recognition (CVPR). (2017)

\end{thebibliography}
\end{document}